%% file: main.tex
\def\year{2021}\relax
\documentclass[letterpaper]{article} 
\usepackage{aaai21}  
\usepackage{times}  
\usepackage{helvet} 
\usepackage{courier}  
\usepackage[hyphens]{url}  
\usepackage{graphicx} 
\usepackage{graphicx}  
\usepackage{natbib}  
\usepackage{caption} 
\usepackage{tikz}
\usepackage{amsmath}
\usepackage{amsthm}
\usepackage{amssymb}
\usepackage{mathtools}
\usepackage{nicefrac}
\usepackage[linesnumbered,algoruled,commentsnumbered]{algorithm2e}
\usepackage{todonotes}

\nocopyright
\title{The Tractability of SHAP-Score-Based Explanations
 over Deterministic and Decomposable Boolean Circuits} 

\author {
        Marcelo Arenas\textsuperscript{\rm 1,2,3},
        Pablo Barceló\textsuperscript{\rm 2,3},
        Leopoldo Bertossi\textsuperscript{\rm 4,3},
        Mikaël Monet\textsuperscript{\rm 5}\\
}
\affiliations {
    \textsuperscript{\rm 1} Department of Computer Science, Universidad Cat\'olica de Chile\\
    \textsuperscript{\rm 2} Institute for Mathematical and Computational Engineering, Universidad Cat\'olica de Chile\\
    \textsuperscript{\rm 3} IMFD Chile\\
    \textsuperscript{\rm 4} Universidad Adolfo Ibáñez, FIC, Chile\\
    \textsuperscript{\rm 5} Inria Lille -- Nord Europe \& UMR 9189 CRIStAL, France\\
}

\input{macros}
\newtheorem{theorem}{Theorem}[section]

\newtheorem{corollary}[theorem]{Corollary}
\newtheorem{definition}[theorem]{Definition}
\newtheorem{example}[theorem]{Example}

\newtheorem{lemma}[theorem]{Lemma}

\allowdisplaybreaks

\usetikzlibrary{positioning}

\begin{document}

\maketitle

\begin{abstract}
\input{abstract}
\end{abstract}

\section{Introduction}
\label{sec:intro}
\input{intro}

\section{Preliminaries}
\label{sec:preliminaries}
\input{preliminaries}

\section{Tractable Computation of the~$\shap$-Score} 
\label{sec:shapscore-d-Ds}
\input{algo}

\section{Limits on the Tractable Computation of the $\shap$-Score}
\label{sec:limits}
\input{limits}

\section{Tractability for the Product Distribution}
\label{sec:prod}
\input{product}

\section{Extensions and Future Work} 
\label{sec:discussion}
\input{discussion}

\section*{Acknowledgments}
We thank Nicholas Mc-Donnell for helping us to find an error in the
reduction from term $\alpha$ to function $H$ in
Section~\ref{subsec:shap-to-ssat}. Pablo Barceló was funded by
Fondecyt grant 1200967.

\bibliography{main}

\newpage

\onecolumn

\appendix

\begin{center}
{\LARGE {\bf Supplementary Material: Technical Appendix}}
\end{center}

\medskip

\section{Encoding Binary Decision Trees and FBDDs into Deterministic and Decomposable Boolean Circuits}
\label{app:KC}
\input{app_KC}

\section{Proof of Theorem \ref{thm:shapscore-d-Ds}}
\label{sec:proof:thm:shapscore-d-Ds}
\input{app-proof-shapscore-d-Ds}

\section{Proof of Lemma~\ref{lem:limits}}
\label{sec:proof:lem:limits}
\input{app-proof-limits}

\section{Proof of Theorem~\ref{thm:shapscore-limits}}
\label{sec:proof:shapscore-limits}
\input{app-proof-shapscore-limits}

\section{Proof of Theorem \ref{thm:shapscore-d-Ds-prod}}
\label{sec:proof:thm:shapscore-d-Ds-prod}
\input{app-proof-shapscore-d-Ds-prod}

\end{document}

%% file: macros.tex
\newcommand\sat{\mathsf{SAT}}

\newcommand\ssat{\mathsf{\#SAT}}

\newcommand\var{\mathsf{var}}

\newcommand\out{\mathsf{out}}
\newcommand\sims{\mathsf{sim}}

\renewcommand\H{\mathsf{H}}
\newcommand\diff{\mathsf{Diff}}

\newcommand{\hrulealg}[0]{\vspace{1mm} \hrule \vspace{1mm}}

\newcommand{\pr}{\mathrm{p}}
\newcommand{\prd}{\Pi_\pr}

\newcommand{\es}{\mathbf{e}}
\newcommand{\eset}{\mathsf{ent}}

\newcommand{\shap}{\mathsf{SHAP}}

\newcommand{\shp}{\#\text{\textsc{P}}}
\newcommand{\asm}{\mathsf{cw}}

\usetikzlibrary{arrows,positioning} 
\tikzset{
	rect/.style={
		rectangle,
		rounded corners,
		draw=black, 
		thick,
		text centered},
	rectw/.style={
		rectangle,
		rounded corners,
		draw=white, 
		thick,
		text centered},
	sq/.style={
		rectangle,
		draw=black, 
		thick,
		text centered},
	sqw/.style={
		rectangle,
		thick,
		text centered},
	arrout/.style={
		->,
		-latex,
		thick,
	},
	arrin/.style={
		<-,
		latex-,
		thick,
		El         },
	arrd/.style={
		<->,
		>=latex,
		thick,
	},
	arrw/.style={
		thick,
	}
}

\tikzset{
	circ/.style={
		circle,
		draw=black, 
		thick,
		text centered,
	},
	circw/.style={
		circle,
		draw=white, 
		thick,
		text centered,
	},
	arrout/.style={
		->,
		-latex,
		thick,
	},
	arrin/.style={
		<-,
		latex-,
		thick,
	},
	arrw/.style={
		-,
		thick,
	},
	arrww/.style={
		-,
		thick,
		draw=white, 
	}
}

\newcommand{\ignore}[1]{}

\newcommand{\feat}[1]{\mbox{\bf \footnotesize #1}}

%% file: abstract.tex
Scores based on Shapley values are 
widely used for providing
explanations to classification results over machine learning models.  A prime
example of this is the influential~$\shap$-score, a version of the Shapley
value that can help explain the result of a learned model on a specific entity
by assigning a score to every feature.
While in general computing Shapley values is a computationally intractable
problem, it has recently been claimed that the~$\shap$-score can be computed in
polynomial time over the class of decision trees. In this paper, we provide a
proof of a stronger result over Boolean models: the $\shap$-score can be
computed in polynomial time over \emph{deterministic and decomposable Boolean
circuits}.
Such
circuits, also known as {\em tractable Boolean circuits}, generalize a
wide range of Boolean circuits and binary decision diagrams classes, including
binary decision trees, Ordered Binary Decision Diagrams (OBDDs) and Free Binary Decision Diagrams (FBDDs).
We also establish the computational limits of the notion of
SHAP-score by observing that, under a mild condition, computing it over a
class of Boolean models is always polynomially as hard as the model counting
problem for that class.
This implies that 
both determinism and decomposability are essential properties for the circuits that we
consider, as removing one or the other renders the problem of
computing the~$\shap$-score intractable (namely,~$\shp$-hard).

%% file: intro.tex
Explainable artificial intelligence
has become an active area of research. Central to it is the 
observation
that artificial intelligence (AI) and machine learning (ML) models cannot always be blindly applied without being able to
interpret and explain their results. 
For example, when someone applies for a loan 
and sees 
their application rejected 
by
an algorithmic decision-making system, the system should be able to
provide an explanation for that decision. 
Explanations can be {\em global} -- focusing on the general input/output relation of the model --, or {\em local} 
-- focusing on how 
features
affect 
the decision of the model
for a specific input. 
Recent literature has strengthened the importance of the latter by showing their ability to provide explanations 
that are often overlooked by global explanations~\cite{molnar}. 

One natural way of providing local explanations for classification models consists in assigning {\em numerical scores} to the feature values of an entity that has gone through the classification process. Intuitively, the higher the score of a feature value, the more relevant it should be considered.
It is in this context that the \emph{$\shap$-score} has been introduced~\cite{lundberg2017unified,lundberg2020local}. 
This recent notion has rapidly gained attention and is becoming
influential. 
There are two properties of the $\shap$-score that 
support its rapid adoption. First, its definition is quite general and can be applied to any kind of classification model.
Second, the definition of the $\shap$-score
is grounded on the well-known {\em Shapley value}~\cite{shapley1953value,roth1988shapley}, that has already been used successfully in 
 several domains of computer science; see, e.g., 
\cite{hunter2010measure,LBKS20,michalak2013efficient,cesari2018application}. 
Thus, $\shap$-scores have a clear, intuitive, combinatorial meaning, 
and inherit all the desirable properties of the Shapley value.

For a given
classifier~$M$, entity~$\es$ and feature~$x$,
the~$\shap$-score~$\shap(M,\es,x)$ intuitively represents the importance of the
feature value~$\es(x)$ to the classification result~$M(\es)$.  In its general
formulation,~$\shap(M,\es,x)$ is a weighted average of differences of expected
values of the outcomes
(c.f.~Section~\ref{sec:preliminaries} for its formal
definition).  Unfortunately, computing quantities that are based on the notion
of Shapley value is in general intractable. Indeed, in many scenarios the
computation turns out to be
$\shp$-hard~\cite{faigle1992shapley,deng1994complexity,LBKS20,BLSSV20}, which makes the
notion difficult to use -- if not impossible -- for practical purposes \cite{AB09}. Therefore,
a natural question is: For what kinds of classification models the computation of the~$\shap$-score can be
done efficiently? This is the subject of this paper.

In this work, we focus on classifiers working with \emph{binary
feature values} (i.e., propositional features that can take the values ~$0$ or~$1$), and
that return~$1$ (accept) or~$0$ (reject) for each entity.  We will call these \emph{Boolean classifiers}. 
The second assumption that we make is that the underlying probability
distribution on the population of entities is what we call a \emph{product
distribution}, where each binary feature~$x$ has a probability~$\pr(x)$ of being
equal to~$1$, independently of the other features.
We note here that the restriction to binary inputs can be relevant
in many practical scenarios where the features are of a propositional nature.

More specifically,
we investigate Boolean classifiers defined as
\emph{deterministic and decomposable Boolean circuits},
a widely studied model
in \emph{knowledge compilation}~\cite{darwiche2001tractability,DBLP:journals/jair/DarwicheM02}.
Such circuits encompass a wide range of Boolean models and binary
decision diagrams classes that are considered in knowledge compilation, and in
AI more generally.  For instance, they generalize {\em binary decision trees}, {\em ordered
binary decision diagrams} (OBDDs), {\em free binary decision diagrams} (FBDDs), and 
{\em deterministic and decomposable negation normal norms} (d-DNNFs) 
\cite{darwiche2001tractability,ACMS20,darwiche2020reasons}.
These circuits are also known under the name of {\em tractable Boolean circuits},
that is used in recent
literature~\cite{shih2019verifying,shi2020tractable,shih2018formal,shih2018symbolic,shih2019smoothing,peharz2020einsum}.
We provide an example of a deterministic and decomposable Boolean circuit next (and give the formal definition in Section~\ref{sec:preliminaries}). 

\medskip

\begin{example} \label{ex:class}
{\em We want to classify papers submitted to a conference as rejected
(Boolean value $0$) or accepted (Boolean value $1$). Papers are
described by features \feat{fg}, \feat{dtr}, \feat{nf} and \feat{na},
which stand for ``follows guidelines", ``deep theoretical result", ``new
framework" and ``nice applications", respectively.  The Boolean classifier
for the papers is given by the Boolean circuit in
Figure \ref{fig:ddbc-exa}. The input of this circuit are the
features \feat{fg}, \feat{dtr}, \feat{nf} and \feat{na}, each of which
can take value either $0$ or $1$, depending on whether the feature is
present~($1$) or absent~($0$). The nodes with labels~$\neg$, $\lor$ or~$\land$ are logic gates, and the associated Boolean value of
each one of them depends on the logical connective represented by its
label and the Boolean values of its inputs. The output 
value
of the circuit is given by 
the top node in the figure.

The Boolean circuit in
Figure \ref{fig:ddbc-exa} is said to be {\em decomposable}, because for each
$\land$-gate, the sets of features of its inputs are pairwise
disjoint. For instance, in the case of the top node in
Figure \ref{fig:ddbc-exa}, the left-hand side input has
$\{\feat{fg}\}$ as its set of features, while its right-hand side
input has $\{\feat{dtr}, \feat{nf}, \feat{na}\}$ as its set of
features, which are disjoint. Also, this circuit is said to be {\em deterministic},
which means that for every $\lor$-gate, two (or more) of its inputs cannot be given
value 1 by the same Boolean assignment for the features. For instance,
in the case of the only $\lor$-gate in Figure \ref{fig:ddbc-exa}, if a
Boolean assignment for the features gives value 1 to its left-hand side
input, then feature \feat{dtr} has to be given value 1 and, thus, such
an assignment gives value $0$ to the right-hand side input of the~$\lor$-gate. In the same way, it can be shown that if a
Boolean assignment for the features gives value 1 to the right-hand side input of this $\lor$-gate, then it gives value $0$ to its left-hand side input.
\qed}
\end{example}

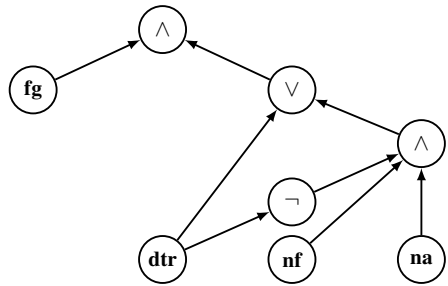
\begin{figure}
\begin{center}
\begin{center}
\resizebox{0.7\columnwidth}{!}{
\begin{tikzpicture}
  \node[circ, minimum size=7mm, inner  sep=-2] (n1) {\feat{dtr}};
  \node[circ, right=12mm of n1, minimum size=7mm] (n2) {\feat{nf}};
  \node[circ, above=1.3mm of n2, minimum size=7mm] (nneg) {$\neg$}
    edge[arrin] (n1);
  \node[circ, right=12mm of n2, minimum size=7mm] (n3) {\feat{na}};
  \node[circ, above=10mm of n3, minimum size=7mm] (n4) {$\land$}
  edge[arrin] (nneg)
  edge[arrin] (n2)
  edge[arrin] (n3);
  \node[circ, above=18mm of n2, minimum size=7mm] (n5) {$\lor$}
  edge[arrin] (n4)
  edge[arrin] (n1);
  \node[circ, above=27mm of n1, minimum size=7mm] (n6) {$\land$}
  edge[arrin] (n5);
  \node[circw, left=12mm of n5, minimum size=7mm] (n5a) {};
  \node[circ, left=12mm of n5a, minimum size=7mm, inner sep=-2] (n0) {\feat{fg}}
  edge[arrout] (n6);
\end{tikzpicture}} 
\end{center}
\caption{A deterministic and decomposable Boolean Circuit as a classifier. \label{fig:ddbc-exa}}
\end{center}
\end{figure}

Readers who are not familiar with knowledge
compilation can simply think about deterministic and decomposable
circuits as a 
tool for establishing in a uniform manner the
tractability of computing $\shap$-scores on several Boolean
classifier classes.
Our main contributions are the following:

\begin{enumerate}
\item
We provide a polynomial time algorithm that computes the $\shap$-score for 
deterministic and decomposable
Boolean circuits, in the special case of \emph{uniform probability distributions} (that is, when each~$\pr(x)$ is~$\frac{1}{2})$. In particular, this provides a precise proof of the claim made in \cite{lundberg2020local} that the $\shap$-score for Boolean classifiers given as decision trees can be computed in polynomial time.
Moreover, 
we also obtain as a corollary that the $\shap$-score for
Boolean classifiers given as
OBDDs, FBDDs and d-DNNFs
can
be computed in polynomial time.
\item
We observe that computing the $\shap$-score on Boolean circuits in a
class is always polynomially as hard as the {\em model counting}
problem
for that class (under a
mild condition).
By using this observation, we obtain that each one of the {\em
determinism} assumption and the {\em decomposability} assumption is necessary
for tractability.
\item
Last, we show that the results above (and most interestingly, the
polynomial-time algorithm) can be extended to the $\shap$-score defined on
product distributions for the entity population.  
\end{enumerate}

Our contributions should be compared to the results obtained in the
contemporaneous paper~\cite{broeck2020tractability}. There, the authors
establish the following theorem: for every class~$\mathcal{C}$ of
classifiers and under product distributions, the problem of computing
the~$\shap$-score for~$\mathcal{C}$ is polynomial-time equivalent to the
problem of computing the expected value for the models in~$\mathcal{C}$.
Since computing expectations is in polynomial time for tractable
Boolean circuits, this in particular implies that computing
the~$\shap$-score is in polynomial time for the circuits that we
consider; in other words, their results capture ours. However, there
is a fundamental difference in the approach taken to show
tractability: their reduction uses multiple oracle calls to the
problem of computing expectations, whereas we provide a more direct
algorithm to compute the~$\shap$-score on these circuits.

Our algorithm for computing the~$\shap$-score could be used in practical scenarios.
Indeed, recently, some classes of classifiers have been compiled
into tractable Boolean circuits.  This is the case, for instance, of Bayesian
Classifiers~\cite{shih2018symbolic}, Binary Neural
Networks~\cite{shi2020tractable}, and Random
Forests~\cite{DBLP:journals/corr/abs-2007-01493}.  The idea is to start with a
Boolean classifier~$M$ given in a formalism that is hard to interpret -- for
instance a Binary neural network -- and to compute a tractable Boolean
circuit~$M'$ that is equivalent to~$M$ (this computation can be expensive). One
can then use~$M'$ and the nice properties of tractable Boolean circuits to
interpret the decisions of the model. Hence, this makes it possible to apply
the results in this paper on the $\shap$-score to those classes of classifiers.

\medskip

{\bf Paper structure.}
We give preliminaries in Section~\ref{sec:preliminaries}.
In
Section~\ref{sec:shapscore-d-Ds}, we prove that~the $\shap$-score can be computed
in polynomial time for deterministic and decomposable Boolean circuits for uniform probability distributions. In
Section~\ref{sec:limits} we establish the limits of the tractable computation of the~$\shap$-score.
Next we show in Section~\ref{sec:prod} that our
results extend to the setting where we consider product distributions.
We conclude and discuss future work in Section~\ref{sec:discussion}.

%% file: preliminaries.tex
\label{sec:prelim} 

\subsection{Entities, distributions and classifiers}
Let~$X$ be a finite set of \emph{features}, also called
\emph{variables}. An \emph{entity} over~$X$ is a function~$\es:X \to
\{0,1\}$.
We denote by~$\eset(X)$ the set of all entities over~$X$.  On this set, we consider
the {\em uniform probability distribution}, i.e., for an event $E \subseteq \eset(X)$, we have that $P(E) := \frac{|E|}{2^{|X|}}$.
We will come back to this assumption in Section~\ref{sec:prod}, where we will consider the more general product distributions (we start with the uniform distribution to ease the presentation).

A \emph{Boolean classifier}~$M$ over~$X$ is 
a
function~$M : \eset(X) \to \{0,1\}$ that maps every entity over~$X$
to~$0$ or~$1$.  We say that~$M$ \emph{accepts} an entity~$\es$
when~$M(\es)=1$, and that it \emph{rejects} it if~$M(\es)=0$. Since we consider $\eset(X)$ to be a probability space, $M$ can be regarded
as a random variable.

\subsection{The $\mathbf{\shap}$-score over Boolean classifiers}

Let $M : \eset(X) \to \{0,1\}$ be a Boolean classifier over the set~$X$ of features.  Given an entity~$\es$ over~$X$ and a
subset~$S \subseteq X$ of features, the set 
$\asm(\es, S) \coloneqq \{ \es' \in \eset(X) \mid \es'(x) =
\es(x) $ for each~$x \in S\}$ contains
those 
entities that coincide with~$\es$ over each feature in~$S$. In other words, $\asm(\es, S)$ is 
the set of entities that are \emph{consistent with} $\es$ on $S$.  Then, given an entity~$\es \in
\eset(X)$ and~$S \subseteq X$, we define the {\em
expected value of $M$ over ~$X \setminus S$ with respect to~$\es$}~as
\begin{align*}
\phi(M,\es,S) \ &\coloneqq \
\mathbb{E}\big[M(\es') \mid \es'\ \in \asm(\es, S) \big].
\end{align*}

Since we consider the uniform distribution over~$\eset(X)$, we have that

\begin{align*}
\phi(M,\es,S)= \sum_{\es' \in \asm(\es,S)} \frac{1}{2^{|X\setminus S|}} M(\es').
\end{align*}
Intuitively,~$\phi(M,\es,S)$ is 
the probability that~$M(\es') = 1$, conditioned on the
inputs~$\es' \in \eset(X)$ to coincide with~$\es$ over each feature
in~$S$. 
This function is then used in the general formula of the Shapley value \cite{shapley1953value,roth1988shapley} to obtain the $\shap$-score for feature values in 
$\es$.
\begin{definition}\label{def:Shapley} 
Given a Boolean classifier~$M$ over a set of features~$X$, an entity~$\es$
over~$X$, and a feature~$x \in X$, the {\em $\shap$ score of
feature~$x$ on~$\es$ with respect to~$M$} is defined~as\newpage
\begin{multline}\label{eq:shapscoredef}
\shap(M,\es,x) \ \coloneqq \ \sum_{S \subseteq X\setminus\{x\}}
\frac{|S|! \, (|X| - |S| - 1)!}{|X|!} \bigg(\\
\phi(M, \es,S \cup \{x\}) - \phi(M, \es,S)\bigg).
\end{multline}
\end{definition}

Thus, $\shap(M,\es,x)$ is a weighted average of the contribution of feature $x$ on $\es$
to the classification result, i.e., of the differences between having it and
not, under all possible permutations of the other feature
values.
Observe that, from this definition, a high positive value
of~$\shap(M,\es,x)$ intuitively means that setting~$x$ to~$\es(x)$
strongly leans the classifier towards acceptance, while a high
negative value of~$\shap(M,\es,x)$ means that setting~$x$ to~$\es(x)$
strongly leans the classifier towards rejection.

\subsection{Deterministic and decomposable Boolean circuits}

A Boolean circuit over a set of variables~$X$ is a directed
acyclic graph~$C$ such that
\begin{enumerate}
\item[(i)] Every node without incoming edges is either a {\em variable
  gate} or a {\em constant gate}. A variable gate is labeled with a
  variable from~$X$, and a constant gate is labeled with either~$0$ or~$1$;

\item[(ii)] Every node with incoming edges is a {\em
  logic gate}, and is labeled with a symbol~$\land$,~$\lor$
  or~$\lnot$. If it is labeled with the symbol~$\lnot$, then
  it has exactly one incoming edge;\footnote{Recall that the fan-in of a gate is the number of its input gates. In our definition of Boolean circuits, we allow
    unbounded fan-in~$\land$- and~$\lor$-gates.}

\item[(iii)] Exactly one node does not have any outgoing edges, and
  this node is called the {\em output gate of~$C$}.
\end{enumerate}
Such a Boolean circuit~$C$ represents a Boolean classifier in the expected way -- we assume the reader to be familiar with Boolean logic --,
and we write~$C(\es)$ for the value in~$\{0,1\}$ of the output gate of~$C$ when we evaluate~$C$ over the entity~$\es$.

Several restrictions of Boolean circuits with good computational
properties have been studied. Let~$C$ be a Boolean circuit over a set
of variables~$X$ and~$g$ a gate of~$C$. The Boolean circuit~$C_g$
over~$X$ is defined by considering the subgraph of~$C$ induced by the
set of gates~$g'$ in~$C$ for which there exists a path from~$g'$
to~$g$ in~$C$. Notice that~$g$ is the output gate of~$C_g$. The
set~$\var(g)$ is defined as the set of variables~$x \in X$ such that
there exists a variable gate with label~$x$ in~$C_g$. Then, an~$\lor$-gate~$g$ of~$C$ is said to be
\emph{deterministic} if for every pair~$g_1$,~$g_2$ of distinct input
gates of~$g$, the Boolean circuits~$C_{g_1}$ and~$C_{g_2}$ are
disjoint in the sense that there is no entity~$\es$ that is accepted
by both~$C_{g_1}$ and~$C_{g_2}$ (that is, there is no entity~$\es \in
\eset(X)$ such that~$C_{g_1}(\es) = C_{g_2}(\es) = 1$). The
circuit~$C$ is called \emph{deterministic} if every~$\lor$-gate of~$C$
is deterministic. An~$\land$-gate~$g$ of~$C$ is said to be
\emph{decomposable} if for every pair~$g_1$,~$g_2$ of distinct input
gates of~$g$, we have that~$\var(g_1) \cap \var(g_2) =
\emptyset$. Then,~$C$ is called \emph{decomposable} if every
$\land$-gate of~$C$ is decomposable.

\medskip

\begin{example} {\em In Example \ref{ex:class}, we explained at an intuitive level why the Boolean circuit in Figure \ref{fig:ddbc-exa} is deterministic and decomposable. By using the terminology defined in the previous paragraph, it can be formally checked that this Boolean circuit indeed satisfies these conditions. \qed}
\end{example}

As mentioned before, 
deterministic and decomposable Boolean circuits
generalize many decision diagrams and Boolean circuits classes.
We refer 
to~\cite{darwiche2001tractability,ACMS20} for detailed studies of
knowledge compilation classes and of their precise relationships.
For the reader's convenience, we explain in the supplementary material
how FBDDs and binary decision trees can be encoded in linear time as
deterministic and decomposable Boolean circuits.

%% file: algo.tex
In this section, we prove our first tractability result, namely, that
computing the~$\shap$-score for Boolean classifiers given as
deterministic and decomposable Boolean circuits can be done in
polynomial time, for uniform probability distributions. Formally:

\begin{theorem}
\label{thm:shapscore-d-Ds} 
The following problem can be solved in polynomial time.
Given as input a deterministic and decomposable Boolean circuit~$C$ over a set of features~$X$,
an entity~$\es:X\to \{0,1\}$, and a feature~$x\in X$,
compute the value~$\shap(C,\es,x)$. 
\end{theorem}

In particular, since binary decision trees, OBDDs, FBDDs and d-DNNFs are all
restricted kinds of deterministic and decomposable circuits, we obtain as a
consequence of Theorem~\ref{thm:shapscore-d-Ds} that this problem is also in
polynomial time for these classes. For instance, for binary decision trees we
obtain:

\begin{corollary}
\label{cor:shapscore-decision-trees} 
The following problem can be solved in polynomial time.
Given as input a binary decision tree~$T$ over a set of features~$X$,
an entity~$\es:X\to \{0,1\}$, and a feature~$x\in X$, compute
the value~$\shap(T,\es,x)$. 
\end{corollary}

The authors of \cite{lundberg2020local} give a
proof of this result, but, unfortunately, with few details to fully
understand it.
Moreover,
it is important to notice that Theorem~\ref{thm:shapscore-d-Ds}
is a nontrivial extension of the result for decision trees, as it is known that
deterministic and decomposable circuits can be exponentially more succinct
than binary decision trees (in fact, than FBDDs)
at representing Boolean
classifiers~\cite{darwiche2001tractability,ACMS20}.

In order to prove Theorem \ref{thm:shapscore-d-Ds}, we need to introduce some notation.
Let~$M$ be a
Boolean classifier over a set of features~$X$.  We
write~$\sat(M)\subseteq \eset(X)$ for the set of entities that are
accepted by~$M$, and~$\ssat(M)$ for the cardinality of this set.
Let~$\es,\es' \in \eset(X)$ be a pair of entities over~$X$. We
define~$\sims(\es,\es') \coloneqq \{x\in
X \mid \es(x)=\es'(x)\}$ to be the set of features on which~$\es$
and~$\es'$ coincide.
Given a Boolean classifier~$M$ over~$X$, an entity~$\es \in \eset(X)$
and a natural number~$k \leq |X|$, we define the
set~$\sat(M,\es,k) \coloneqq \sat(M) \cap \{ \es'
\in \eset(X) \mid |\sims(\es,\es')| = k\}$, in other words, the set of
entities~$\es'$ that are accepted by~$M$ and which coincide with~$\es$
in exactly~$k$ features.
Naturally, we write~$\ssat(M,\es,k)$ for the size of~$\sat(M,\es,k)$.

\medskip

\begin{example} {\em
Let~$M$ be the Boolean classifier represented by the circuit
in Example~\ref{ex:class}.
Then~$\sat(M)$ is the set containing all papers that are accepted
according to
$M$,
so that~$\ssat(M) = 5$.  Now, consider the entity~$\es$ such that
$\es(\feat{fg}) = 1$, $\es(\feat{dtr}) = 1$, $\es(\feat{nf}) = 0$ and
$\es(\feat{na}) = 1$.
Then one can check that $\ssat(M,\es,0) = 0$,
$\ssat(M,\es,1) = 0$, $\ssat(M,\es,2) = 2$, $\ssat(M,\es,3) = 2$ and
$\ssat(M,\es,4) = 1$.
\qed }
\end{example}

Our proof of Theorem~\ref{thm:shapscore-d-Ds} is technical and is
divided into two modular parts. The first part, which is developed in
Section \ref{subsec:shap-to-ssat}, consists in showing that the
problem of computing~$\shap(\cdot,\cdot,\cdot)$ can be reduced in
polynomial time to that of computing~$\ssat(\cdot,\cdot,\cdot)$. This
part of the proof is a sequence of formula manipulations, and it only
uses the fact that deterministic and decomposable circuits can
be efficiently \emph{conditioned} on a variable value (to be defined
in Section \ref{subsec:shap-to-ssat}).
In the second part of the proof, which is developed in
Section \ref{subsec:ssatle}, we show that
computing~$\ssat(\cdot,\cdot,\cdot)$ can be done in polynomial time
for deterministic and decomposable Boolean circuits.
It is in this part that the properties 
of deterministic and decomposable
circuits are really used.

\subsection{Reducing 
$\shap(\cdot,\cdot,\cdot)$ to~$\ssat(\cdot,\cdot,\cdot)$}
\label{subsec:shap-to-ssat}

In this section, we show that for deterministic and decomposable
Boolean circuits, the computation of the~$\shap$-score can be reduced
in polynomial time to the computation
of~$\ssat(\cdot,\cdot,\cdot)$. To achieve this, we will need two more
definitions. Let~$M$ be a Boolean classifier over a set of
features~$X$ and~$x\in X$, and let Boolean
classifiers~$M_{+x}:\eset(X\setminus \{x\}) \to \{0,1\}$
and~$M_{-x}:\eset(X\setminus \{x\}) \to \{0,1\}$ be defined as
follows. For~$\es \in \eset(X\setminus \{x\})$, we 
write~$\es_{+x}$ and~$\es_{-x}$ the entities over~$X$ such that
$\es_{+x}(x)=1$,~$\es_{-x}(x)=0$ and $\es_{+x}(y) = \es_{-x}(y)
= \es(y)$ for every~$y \in X \setminus \{x\}$. Then define
$M_{+x}(\es) \coloneqq M(\es_{+x})$ and~$M_{-x}(\es) \coloneqq
M(\es_{-x})$.  In the literature,~$M_{+x}$ (resp.,~$M_{-x}$) is called
the \emph{conditioning by~$x$ (resp., by~$\lnot x$) of~$M$}.  
Conditioning can be done in linear time for a Boolean
circuit~$C$ by replacing every gate with label~$x$ by a constant gate
with label~$1$ (resp.,~$0$).
We write~$C_{+x}$ (resp.,~$C_{-x}$) for the Boolean circuit obtained via
this transformation. One can easily check that, if~$C$ is 
deterministic and decomposable, then~$C_{+x}$
and~$C_{-x}$ are deterministic and decomposable as well.

We now introduce the second definition needed for the proof.
For a Boolean classifier~$M$ over a set of variables~$X$, an
entity~$\es\in \eset(X)$ and an integer~$k\leq |X|$, we define
\begin{align}
\label{eq:def-H}
\H(M,\es, k) \ \coloneqq \ \sum_{\substack{S \subseteq X\\|S|=k}} \, \, \, \sum_{\es'\in \asm(\es,S)} M(\es').
\end{align}
We first explain
how computing~$\shap(\cdot,\cdot,\cdot)$ can be reduced in polynomial
time to the problem of computing~$\H(\cdot,\cdot,\cdot)$, and then
how computing~$\H(\cdot,\cdot,\cdot)$ can be reduced in polynomial
time to computing~$\ssat(\cdot,\cdot,\cdot)$.

\medskip

{\bf Reducing from~$\shap(\cdot,\cdot,\cdot)$
to~$\H(\cdot,\cdot,\cdot)$.}  We need to compute~$\shap(C,\es,x)$, for
a given deterministic and decomposable circuit~$C$ over a set of
variables~$X$, entity~$\es\in \eset(X)$, and feature~$x\in X$. Let~$n = |X|$, and define
\[\diff_k(C,\es,x) \ \coloneqq \ \sum_{\substack{S\subseteq X\setminus \{x\}\\|S|=k}} (\phi(C,\es,S\cup \{x\}) - \phi(C,\es,S)).\]
Then by the definition of the~$\shap$-score
in \eqref{eq:shapscoredef}, we have:
\begin{align*}
\shap(C,\es,x) \ = \
\sum_{k=0}^{n-1} \frac{k!(n-k-1)!}{n!} \diff_k(C,\es,x).
\end{align*}
Observe that all arithmetical terms (such as~$k!$ or~$n!$) can be computed in polynomial time: this is simply because~$n$ is given in unary, as it is bounded by the
size of the circuit.
Therefore, it is enough to show how to compute in polynomial time the
quantities~$\diff_k(C,\es,x)$ for each~$k\in \{0,\ldots,n-1\}$, as~$n = |X|$ is bounded by the size of the input~$(C,\es,x)$.
By definition
of~$\phi(\cdot,\cdot,\cdot)$, we have that~$\diff_k(C,\es,x) = \alpha - \beta$, where:
\begin{eqnarray*}
\alpha & = & \sum_{\substack{S\subseteq X\setminus \{x\}\\|S|=k}} \frac{1}{2^{n-(k+1)}} \sum_{\es' \in \asm(\es,S\cup \{x\})} C(\es')\\
\beta & = & \sum_{\substack{S\subseteq X\setminus \{x\}\\|S|=k}} \frac{1}{2^{n-k}} \sum_{\es' \in \asm(\es,S)} C(\es').
\end{eqnarray*}
Next we show how the computation of~$\alpha$ and~$\beta$ can be
reduced in polynomial-time to the computation of
$\H(\cdot,\cdot,\cdot)$. For an entity~$\es \in \eset(X)$
and~$S\subseteq X$, let~$\es_{|S}$ be the entity over~$S$ that is
obtained by restricting~$\es$ to the domain~$S$ (that
is, formally~$\es_{|S} \in \eset(S)$ and~$\es_{|S}(y) \coloneqq \es(y)$ for
every~$y \in S$). Then, starting with~$\beta$, we have that:
\begin{align*}
\beta \ =& \ \sum_{\substack{S\subseteq X\setminus \{x\}\\|S|=k}} \frac{1}{2^{n-k}} \sum_{\es' \in \asm(\es,S)} C(\es')\\
=& \bigg[\sum_{\substack{S\subseteq X\setminus \{x\}\\|S|=k}} \frac{1}{2^{n-k}} \sum_{\substack{\es' \in \asm(\es,S)\\\es'(x)=1}} C(\es')\bigg]\\
& \hspace{25pt} + \bigg[\sum_{\substack{S\subseteq X\setminus \{x\}\\|S|=k}} \frac{1}{2^{n-k}} \sum_{\substack{\es' \in \asm(\es,S)\\\es'(x)=0}} C(\es')\bigg]\\
=& \bigg[ \frac{1}{2^{n-k}}\sum_{\substack{S\subseteq X\setminus \{x\}\\|S|=k}} \sum_{\substack{\es'' \in \asm(\es_{|X\setminus\{x\}},S)}} C_{+x}(\es'')\bigg]\\
& \hspace{25pt} + \bigg[\frac{1}{2^{n-k}} \sum_{\substack{S\subseteq X\setminus \{x\}\\|S|=k}} \sum_{\substack{\es'' \in \asm(\es_{|X\setminus\{x\}},S)}} C_{-x}(\es'')\bigg]\\
=& \frac{1}{2^{n-k}} \bigg(\H(C_{+x},\es_{|X\setminus\{x\}},k) +  \H(C_{-x}, \es_{|X\setminus\{x\}},k)\bigg).
\end{align*}
The last equality is obtained 
by using the definition
of~$\H(\cdot,\cdot,\cdot)$. 
A similar analysis allows us to conclude that: 
\[
\alpha \ = \ \begin{cases}
 {\displaystyle \frac{1}{2^{n-(k+1)}} \H(C_{+x},\es_{|X\setminus\{x\}},k)}, & \text{if } \es(x)=1\\
 {\displaystyle \frac{1}{2^{n-(k+1)}} \H(C_{-x},\es_{|X\setminus\{x\}},k)}, & \text{if } \es(x)=0
 \end{cases}.
\]

Hence, if we can compute in polynomial
time~$\H(\cdot,\cdot,\cdot)$ for deterministic and decomposable Boolean
circuits, then we can compute~$\alpha$ and~$\beta$ in polynomial time (because~$C_{+x}$
and~$C_{-x}$ can be computed in linear time from~$C$, and they
are deterministic and decomposable as well).
Thus, we can compute~$\diff_k(C,\es,x)$ in polynomial time for
each~$k\in \{0,\ldots,n-1\}$ and, hence,~$\shap(C,\es,x)$ as well.
In conclusion,~$\shap(C,\es,x)$ can be computed in
polynomial time if there is a polynomial-time algorithm to
compute~$\H(\cdot,\cdot,\cdot)$ for deterministic and decomposable
Boolean circuits.

\medskip

{\bf Reducing from~$\H(\cdot,\cdot,\cdot)$
to~$\ssat(\cdot,\cdot,\cdot)$.}  We now show that
computing~$\H(\cdot,\cdot,\cdot)$ can be reduced in polynomial time to
computing~$\ssat(\cdot,\cdot,\cdot)$.  Given as input a deterministic
and decomposable circuit~$C$ over a set of variables~$X$, an
entity~$\es\in \eset(X)$, and an integer~$k \leq |X|$, recall the definition of~$\H(C,\es,x)$ in \eqref{eq:def-H}.
Then consider an entity~$\es'' \in \eset(X)$ and reason about
how many times~$\es''$ will occur as a summand in the
expression \eqref{eq:def-H}.  First of all, it is clear that
if~$|\sims(\es,\es'')| < k$, then~$\es''$ will not appear in the sum;
this is because if~$\es' \in \asm(\es,S)$ for some~$S \subseteq X$
such that~$|S| = k$, then~$S \subseteq
\sims(\es,\es')$ and, thus,~$k \leq |\sims(\es,\es')|$.
Now, how many times does an entity~$\es''\in \eset(X)$ such
that~$|\sims(\es,\es'')| \geq k$ occur as a summand in the expression?
The answer is simple: once per~$S\subseteq \sims(\es,\es'')$ of
size~$k$.  Since there are~$\binom{|\sims(\es,\es'')|}{k}$ such sets~$S$, we
obtain that~$\H(C,\es,k)$ is equal to 
\begin{align*}
&\sum_{\substack{\es''\in \eset(X)\\ |\sims(\es,\es'')|\geq k}} \binom{|\sims(\es,\es'')|}{k} \cdot C(\es'')\\
&= \ \sum_{\substack{\es''\in \sat(C)\\ |\sims(\es,\es'')|\geq k}} \binom{|\sims(\es,\es'')|}{k}\\
&= \ \sum_{\ell=k}^n  \sum_{\substack{\es''\in \sat(C)\\ |\sims(\es,\es'')| = \ell}} \binom{|\sims(\es,\es'')|}{k}\\
&= \ \sum_{\ell=k}^n  \binom{\ell}{k}\sum_{\substack{\es''\in \sat(C)\\ |\sims(\es,\es'')|= \ell}} 1 \ \ = \ \ \sum_{\ell=k}^n \binom{\ell}{k} \cdot \ssat(C,\es,\ell),
\end{align*}
with the last equality being obtained by using the definition of
$\ssat(\cdot,\cdot,\cdot)$.  This concludes the reduction of this
section and, hence, the first part of the proof.

\subsection{Computing~$\ssat(\cdot,\cdot,\cdot)$ in polynomial time}
\label{subsec:ssatle}

We now take care of the second part of the proof of
Theorem~\ref{thm:shapscore-d-Ds}, i.e., proving that
computing~$\ssat(\cdot,\cdot,\cdot)$ for deterministic and
decomposable Boolean circuits can be done in polynomial time.
To do this, given a deterministic and
decomposable Boolean circuit~$C$, we first perform
two preprocessing steps on~$C$, which will simplify
the proof.

\begin{itemize} 
\item 
{\bf Rewriting to fan-in at most 2.} First, we modify the circuit~$C$
so that the fan-in of every~$\lor$- and~$\land$-gate is at
most~$2$. This can simply be done in linear time by rewriting
every~$\land$-gate (resp., and~$\lor$-gate) of fan-in~$m > 2$ with a
chain of~$m-1$ $\land$-gates (resp.,~$\lor$-gates) of fan-in~$2$. It
is clear that the resulting Boolean circuit is deterministic and
decomposable. Hence, from now on we assume that the fan-in of
every~$\lor$- and~$\land$-gate of~$C$ is at most~$2$.

\item 
{\bf Smoothing the circuit.} A deterministic and decomposable circuit~$C$ is
\emph{smooth}~\cite{darwiche2001tractability,shih2019smoothing} if for
every~$\lor$-gate~$g$ and input gates~$g_1,g_2$ of~$g$, we have that~$\var(g_1)
= \var(g_2)$, and we call such an~$\lor$-gate smooth. A standard construction
allows to transform in polynomial time a deterministic and decomposable Boolean
circuit~$C$ into an equivalent smooth deterministic and decomposable Boolean
circuit, and where each gate has fan-in at most~2. Thus, from now on we also
assume that~$C$ is smooth.  We illustrate how the construction works in Example
\ref{ex:loooooong} .  Full details can be found in the supplementary material
(namely, in Section~\ref{subsec:H}, paragraph
\emph{Smoothing the circuit}). 
\end{itemize}

\medskip

We have all the ingredients to prove that~$\ssat(\cdot,\cdot,\cdot)$
can be computed in polynomial time. Let~$C$ be a deterministic and
decomposable Boolean circuit
over a set of variables~$X$,
$\es\in \eset(X)$, $\ell$ a natural number such that~$\ell \leq |X|$
and~$n = |X|$. For a gate~$g$ of~$C$, let~$R_g$ be the Boolean circuit
over~$\var(g)$ that is defined by considering the subgraph of~$C$
induced by the set of gates~$g'$ in~$C$ for which there exists a path
from~$g'$ to~$g$ in~$C$. Notice that~$R_g$ is a deterministic and
decomposable Boolean circuit with output gate~$g$.  Moreover, for a
gate~$g$ and natural number~$k \leq |\var(g)|$,
define~$\alpha_g^k \coloneqq \ssat(R_g,\es_{|\var(g)},k)$, which we
recall is the number of entities~$\es'\in \eset(\var(g))$ such
that~$\es'$ satisfies~$R_g$ and~$|\sims(\es_{|\var(g)},\es')|=k$. We
will show how to compute all the values~$\alpha_g^k$ for every
gate~$g$ of~$C$ and~$k\in \{0,\ldots,|\var(g)|\}$ in polynomial
time. This will conclude the proof since, for the output gate~$g_\out$
of~$C$, we have that~$\alpha_{g_\out}^\ell = \ssat(C,\es,\ell)$.  Next
we explain how to compute these values in a bottom-up manner.

\begin{description}
\item[Variable gate.]
$g$ is a variable gate with label~$y \in X$, so
that~$\var(g)=\{y\}$. Then~$\alpha^0_g = 1 - \es(y)$
and~$\alpha^1_g = \es(y)$.

\item[Constant gate.]
$g$ is a constant gate with label~$a\in \{0,1\}$.  Then~$\var(g)=\emptyset$ and~$\alpha_g^0 =
a$.\footnote{We recall the mathematical
convention that there is a unique function with the
empty domain and, hence, a unique entity
over~$\emptyset$.}

\item[$\lnot$-gate.]
$g$ is a~$\lnot$-gate with input gate~$g'$. Then~$\var(g)=\var(g')$,
and the values~$\alpha_{g'}^k$ for~$k\in\{0,\ldots,|\var(g)|\}$
have already been computed. Fix~$k \in\{0,\ldots,|\var(g)|\}$. Since~$\binom{|\var(g)|}{k}$ is equal to the number of entities~$\es' \in \eset(\var(g))$ such
that~$|\sims(\es_{|\var(g)},\es')|=k$, we have
that
\begin{align*}
\alpha_g^k \ = \ \binom{|\var(g)|}{k}
-\alpha_{g'}^k.
\end{align*}
Therefore, given that~$\binom{|\var(g)|}{k}$
can be computed in polynomial time since $k \leq |\var(g)| \leq n =
|X|$, we have an efficient way to compute~$\alpha_g^k$.

\item[$\lor$-gate.]
$g$ is an~$\lor$-gate. By assumption,~$g$ is
deterministic, smooth and has fan-in at most 2.  If~$g$ has only one
input~$g'$, then clearly~$\var(g)=\var(g')$ and~$\alpha_g^k
= \alpha_{g'}^k$ for every~$k\in \{0,\ldots,|\var(g)|\}$. Thus, assume
that~$g$ has exactly two input gates~$g_1$ and~$g_2$, and recall
that~$\var(g_1) = \var(g_2) = \var(g)$, because~$g$ is
smooth. Also, recall that~$\alpha_{g_1}^k$
and~$\alpha_{g_2}^k$, for each~$k\in\{0,\ldots,|\var(g)|\}$, have
already been computed. Fix~$k\in\{0,\ldots,|\var(g)|\}$. Given
that~$g$ is deterministic and smooth, we have that
$\sat(R_g) \ = \ \sat(R_{g_1})  \cup \sat(R_{g_2})$,
where~$\sat(R_{g_1}) \cap \sat(R_{g_2}) = \emptyset$. By intersecting
these three sets with the set $\{ \es' \in \var(g) \mid
|\sims(\es_{|\var(g)},\es')|=k\}$, we obtain
that $\sat(R_g,\es_{|\var(g)},k) = 
\sat(R_{g_1},\es_{|\var(g)},k) \cup \sat(R_{g_2},\es_{|\var(g)},k)$,
where
$\sat(R_{g_1},\es_{|\var(g)},k) \cap \sat(R_{g_2},\es_{|\var(g)},k)
= \emptyset$. Hence:
\begin{multline*}
\ssat(R_g,\es_{|\var(g)},k) \ = \\
\ssat(R_{g_1},\es_{|\var(g)},k)  + \ssat(R_{g_2},\es_{|\var(g)},k),
\end{multline*}
or, in other words, we have that 
$\alpha_g^k = \alpha_{g_1}^k + \alpha_{g_2}^k$. Hence, we have an
efficient way to compute~$\alpha_g^k$.
\item[$\land$-gate.]
$g$ is an~$\land$-gate. By assumption, recall that~$g$ is decomposable
and has fan-in at most 2.  If~$g$ has only one input~$g'$, then
clearly~$\var(g)=\var(g')$ and~$\alpha_g^k = \alpha_{g'}^k$ for
every~$k\in \{0,\ldots,|\var(g)|\}$. Thus, assume that~$g$ has exactly
two input gates~$g_1$ and~$g_2$. Recall then that the
values~$\alpha_{g_1}^i$ and~$\alpha_{g_2}^j$, for each
$i\in\{0,\ldots,|\var(g_1)|\}$ and $j\in \{0,\ldots,|\var(g_2)|\}$,
have already been computed. Fix~$k \in \{0,\ldots,|\var(g)|\}$. Given
that~$g$ is a decomposable~$\land$-gate, in this case it is possible
to prove that:
\begin{align}
\label{eq:and-gate}
\alpha_g^k \ = \ \sum_{\substack{i \in \{0, \ldots, |\var(g_1)|\} \\ j \in \{0, \ldots, |\var(g_2)|\}
\\ i+j=k}} 
\alpha_{g_1}^i \cdot \alpha_{g_2}^j.
\end{align}
The complete proof of this property can be found in Appendix~\ref{sec:proof:thm:shapscore-d-Ds}. Therefore, as in the previous cases, we conclude that there is an efficient way to compute~$\alpha_g^k$.
\end{description}
This concludes the proof that~$\ssat(\cdot,\cdot,\cdot)$
can be computed in polynomial time for deterministic and decomposable Boolean circuits and, hence, the proof of
Theorem~\ref{thm:shapscore-d-Ds}.

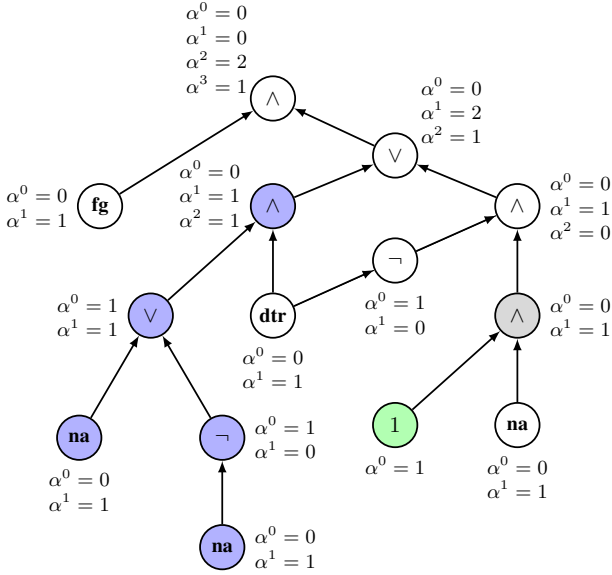
\begin{figure}
\begin{center}
\begin{center}
\resizebox{\columnwidth}{!}{
\begin{tikzpicture}
  \node[circ, minimum size=7mm, inner sep=-2] (n1) {\feat{dtr}};
  \node[sqw, below=0mm of n1, text width=10mm] (vn1) {{\footnotesize $\alpha^0 = 0$ $\alpha^1 = 1$}};
  \node[circ, above=10mm of n1, minimum size=7mm, inner sep=-2, fill=blue!30] (n1a) {$\land$}
  edge[arrin] (n1);
    \node[sqw, left=0mm of n1a, text width=10mm] (vn1a) {{\footnotesize $\alpha^0 = 0$ $\alpha^1 = 1$ $\alpha^2 = 1$ \phantom{$\alpha^2=1$}}};
  \node[circ, left=12mm of n1, minimum size=7mm, inner sep=-2, fill=blue!30] (n1o) {$\lor$}
  edge[arrout] (n1a);
    \node[sqw, left=0mm of n1o, text width=10mm] (vn1o) {{\footnotesize $\alpha^0 = 1$ $\alpha^1 = 1$}};
  \node[circw, below=12mm of n1o, minimum size=7mm, inner sep=-2] (n1oa) {};
  \node[circ, left=4mm of n1oa, minimum size=7mm, inner sep=-2, fill=blue!30] (n1ol) {\feat{na}}
  edge[arrout] (n1o);
  \node[sqw, below=0mm of n1ol, text width=10mm] (vn1ol) {{\footnotesize $\alpha^0 = 0$ $\alpha^1 = 1$}};
  \node[circ, right=4mm of n1oa, minimum size=7mm, inner sep=-2, fill=blue!30] (n1or) {$\neg$}
  edge[arrout] (n1o);
  \node[sqw, right=0mm of n1or, text width=10mm] (vn1or) {{\footnotesize $\alpha^0 = 1$ $\alpha^1 = 0$}};
  \node[circ, below=10mm of n1or, minimum size=7mm, inner sep=-2, fill=blue!30] (n1oru) {\feat{na}}
  edge[arrout] (n1or);
    \node[sqw, right=0mm of n1oru, text width=10mm] (vn1oru) {{\footnotesize $\alpha^0 = 0$ $\alpha^1 = 1$}};
  \node[circ, left=20mm of n1a, minimum size=7mm, inner sep=-2] (n0) {\feat{fg}};
  \node[sqw, left=0mm of n0, text width=10mm] (vn0) {{\footnotesize $\alpha^0 = 0$ $\alpha^1 = 1$}};
  \node[circw, right=12mm of n1, minimum size=7mm] (n2) {};
  \node[circ, right=12mm of n2, minimum size=7mm, fill=gray!30] (n3) {$\land$};
  \node[sqw, right=0mm of n3, text width=10mm] (vn3) {{\footnotesize $\alpha^0 = 0$ $\alpha^1 = 1$}};
  \node[circ, below=10mm of n3, minimum size=7mm] (n3r) {\feat{na}}
  edge[arrout] (n3);
  \node[sqw, below=0mm of n3r, text width=10mm] (vn3r) {{\footnotesize $\alpha^0 = 0$ $\alpha^1 = 1$}};
  \node[circ, left=12mm of n3r, minimum size=7mm, fill=green!30] (n3l) {$1$}
  edge[arrout] (n3);
    \node[sqw, below=0mm of n3l] (v3l) {{\footnotesize $\alpha^0 = 1$}};
  \node[circ, above=10mm of n3, minimum size=7mm] (n4) {$\land$}
  edge[-] node[fill=white, circw, minimum size=7mm] (en4) {$\neg$} (n1)
  edge[arrin] (n3);
  \node[circ, above=18.7mm of n3l, minimum size=7mm] (nnegaux) {}
    edge[arrin] (n1)
    edge[arrout] (n4);
  \node[sqw, below=0mm of en4, text width=10mm] (ven4) {{\footnotesize $\alpha^0 = 1$ $\alpha^1 = 0$}};
  \node[sqw, right=0mm of n4, text width=10mm] (vn4) {{\footnotesize $\alpha^0 = 0$ $\alpha^1 = 1$ $\alpha^2 = 0$}};

  \node[circ, above=18mm of n2, minimum size=7mm] (n5) {$\lor$}
  edge[arrin] (n4)
  edge[arrin] (n1a);
  \node[sqw, above right = -2mm and 0mm of n5, text width=10mm] (vn5) {{\footnotesize $\alpha^0 = 0$ $\alpha^1 = 2$ $\alpha^2 = 1$ }};
  \node[circ, above=27mm of n1, minimum size=7mm] (n6) {$\land$}
  edge[arrin] (n0) edge[arrin] (n5);
  \node[sqw, above left = -3mm and 0mm of n6, text width=10mm] (vn6) {{\footnotesize $\alpha^0 = 0$ $\alpha^1 = 0$ $\alpha^2 = 2$ $\alpha^3=1$}};
\end{tikzpicture} 
} 
\end{center}
\caption{
Execution of our algorithm
to compute
$\ssat(\cdot$, $\cdot, \cdot)$ over the Boolean circuit~$C_{+\feat{nf}}$ from Example~\ref{ex:loooooong}. \label{fig:ddbc-sat-alg-exa}}
\end{center}
\end{figure}

\medskip

\begin{example}
\label{ex:loooooong} 
{\em 
We illustrate 
how the
algorithm 
for computing the $\shap$-score
operates on the Boolean circuit~$C$ given
in Example \ref{ex:class}. Recall that~$C$ is defined over 
$X = \{\feat{fg}, \feat{dtr}, \feat{nf}, \feat{na}\}$, and assume we
want to compute~$\shap(C, \es, \feat{nf})$ for the entity~$\es$ with 
$\es(x) = 1$ for each~$x \in X$. By the polynomial time
reductions shown in Section \ref{subsec:shap-to-ssat}, to
compute~$\shap(C, \es, \feat{nf})$ it suffices to compute
$\H(C_{-\feat{nf}}, \es_{|X\setminus\{\feat{nf}\}}, \ell)$ and
$\H(C_{+\feat{nf}}, \es_{|X\setminus\{\feat{nf}\}}, \ell)$ for each
$\ell \in \{0, \ldots, 3\}$, which in turn reduces to the computation
of $\ssat(C_{-\feat{nf}}, \es_{|X\setminus\{\feat{nf}\}}, \ell)$ and
$\ssat(C_{+\feat{nf}}, \es_{|X\setminus\{\feat{nf}\}}, \ell)$ for each
$\ell \in \{0, \ldots, 3\}$. In what follows, we show how to compute
the values
$\ssat(C_{+\feat{nf}}, \es_{|X\setminus\{\feat{nf}\}}, \ell)$.

For the sake of presentation, let~$D :=C_{+\feat{nf}}$
and~$\es^\star
= \es_{|X\setminus\{\feat{nf}\}}$, so that we need to compute
$\ssat(D, \es^\star, \ell)$ for each~$\ell \in \{0, \ldots, 3\}$.
Notice that the values to be computed 
are
$\ssat(D, \es^\star, 0) = 0$, $\ssat(D, \es^\star, 1) = 0$, $\ssat(D, \es^\star, 2) =
2$ and $\ssat(D, \es^\star, 3) = 1$.
To compute~$\ssat(D, \es^\star, \ell)$, we first need to replace
feature~$\feat{nf}$ by constant~$1$ in~$C$ to generate
$D = C_{+\feat{nf}}$, and then we need to transform~$D$ into
a Boolean circuit that is smooth and where each gate has fan-in at
most~2. The result of this process is shown in
Figure \ref{fig:ddbc-sat-alg-exa}, where the green node is added when
replacing feature~$\feat{nf}$ by constant~$1$, the gray node is added
to satisfy the restriction that each gate has fan-in at most 2, and
the blue nodes are added to have a smooth Boolean circuit.

The algorithm to compute~$\ssat(D, \es^\star, \ell)$ runs in a bottom-up fashion on 
the Boolean circuit, computing for each gate~$g$
the values~$\alpha^k_g$ for~$k \in \{0, \ldots, |\var(g)|\}$.
We show these values next to each node in
Figure \ref{fig:ddbc-sat-alg-exa}, but omitting
gate subscripts. For instance, for a variable gate~$g$ with label
$\feat{na}$, we have that $\var(g) = \{\feat{na}\}$, $\alpha^0_g = 0$
and $\alpha^1_g = 1$, given that $\es^\star_{|\var(g)}(\feat{na})
= \es^\star(\feat{na}) = 1$. Notice that for the output gate~$g_\out$ of the Boolean circuit, 
which is its top gate, we have that $\ssat(D, \es^\star, \ell)
= \alpha^\ell_{g_\out}$ for each~$\ell \in \{0, \ldots, 3\}$, which
were the values to be computed. \qed}
\end{example}

%% file: limits.tex
We have shown that the $\shap$-score can be computed in polynomial time for deterministic and decomposable circuits. 
A natural question, then, is whether both determinism and decomposability are necessary for this positive result to hold. 
In this section we show this to be case, at least under standard complexity assumptions. 
Recall that $\shp$ consists of the class of functions that can be defined by counting 
the number of accepting paths of a non-deterministic Turing machine that works in polynomial time. The notion of hardness for 
the class $\shp$ is defined in terms of polynomial time {\em Turing reductions}. Under widely-held complexity assumptions, $\shp$-hard problems 
cannot be solved in polynomial time~\cite{AB09}.  We can then prove the following: 

\begin{theorem}
\label{thm:shapscore-limits}
The following problems are~$\shp$-hard.  
\begin{enumerate} 
\item Given as input a decomposable Boolean circuit~$C$
over a set of features~$X$, an entity~$\es:X\to \{0,1\}$, and a feature~$x\in
X$, compute the value~$\shap(C,\es,x)$.
\item 
Given as input a deterministic Boolean circuit~$C$
over a set of features~$X$, an entity~$\es:X\to \{0,1\}$, and a feature~$x\in
X$, compute the value~$\shap(C,\es,x)$.
\end{enumerate} 
\end{theorem}

To prove Theorem \ref{thm:shapscore-limits}, we start by showing
that there is a polynomial-time reduction from the 
problem of computing the number of entities that satisfy $M$, for $M$ an arbitrary Boolean classifier, 
to the problem of computing the $\shap$-score over $M$. This holds under the mild condition that $M(\es)$ can be computed in polynomial time for an input entity $\es$, which is satisfied for all the Boolean
circuits and binary decision diagrams classes considered in this
paper.
The proof of this result follows from well-known properties of Shapley values. (A closely related result can be found as
Theorem~5.1 in~\cite{BLSSV20}). 

\begin{lemma}
\label{lem:limits}
Let~$M$ be a Boolean classifier over a set of features~$X$. Then for every~$\es \in \eset(X)$ we have:
\begin{align*}
\ssat(M) \ = \ 2^{|X|} \bigg( M(\es) - \sum_{x\in X} \shap(M,\es,x) \bigg).
\end{align*}
\end{lemma}

We prove~Lemma~\ref{lem:limits} in the supplementary material.
Item~(1) in Theorem \ref{thm:shapscore-limits} follows then by the following two facts: (a) Counting the number of entities that satisfy 
a DNF formula is a~$\shp$-hard problem~\cite{provan1983complexity}, and~(b) 
DNF formulae are particular kinds of decomposable Boolean circuits. Analogously, 
item~(2) in Theorem \ref{thm:shapscore-limits} can be obtained from the following two facts: 
(a) Counting the number of entities that satisfy 
a 3-CNF formula is a~$\shp$-hard problem, and
(b)
from every 3-CNF formula $\psi$, we can build 
in polynomial time an equivalent deterministic Boolean circuit $C_\psi$. Details can be found in the supplementary material.

%% file: product.tex
In Section 
\ref{sec:prelim}, 
we introduce the uniform distribution, and used it so far as a basis for the $\shap$-score. 
Another probability space that is often considered on $\eset(X)$ is the \emph{product distribution},
defined as follows. Let~$\pr:X \to [0,1]$ be a function that associates to
every feature~$x \in X$ a value~$\pr(x) \in [0,1]$; intuitively, the probability that $x$ takes value $1$. Then, the
\emph{product distribution generated by~$\pr$} is the
probability distribution~$\prd$ over~$\eset(X)$ such that, for every~$\es\in
\eset(X)$, 
\[\prd(\es) \ \ \coloneqq \ \ \bigg(\prod_{\substack{x\in X\\ \es(x)=1}} \pr(x)\bigg) \cdot \bigg(\prod_{\substack{x\in X\\ \es(x)=0}} (1-\pr(x))\bigg).\]
That is, the product distribution that is determined by pre-specified marginal distributions, and that makes the features take values independently from each other.
Observe the effect of
the probability distribution on the $\shap$-score: intuitively, the higher the
probability of an entity, the more impact this entity will have on the
computation. This can be used, for instance, to avoid bias in the
explanations~\cite{lundberg2017unified,BLSSV20}.

Notice that the uniform space is a special case
of product space, with~$\prd$ invoking~$\pr(x) := \nicefrac{1}{2}$ for every~$x\in X$.
Thus, our hardness results from Theorem~\ref{thm:shapscore-limits} also hold in the case where the probabilities~$\pr(x)$ are given as input.
What is more interesting is the fact that our tractability result from Theorem~\ref{thm:shapscore-d-Ds} extends to product distributions. Formally:

\begin{theorem}
\label{thm:shapscore-d-Ds-prod}
The following problem can be solved in polynomial time.  Given as input a
deterministic and decomposable circuit~$C$ over a set of features~$X$, rational
probability values~$\pr(x)$ for every feature~$x\in X$, an entity~$\es:X\to
\{0,1\}$, and a feature~$x\in X$, compute the value~$\shap(C,\es,x)$ under the probability distribution~$\prd$.
\end{theorem}

The proof of Theorem~\ref{thm:shapscore-d-Ds-prod} is more involved
than that of Theorem~\ref{thm:shapscore-d-Ds}, and is provided in the
supplementary material.
In particular, the main difficulty is that $\phi(M,\es,S)$ is no longer
equal to $\sum_{\es' \in \asm(\es,S)}
\frac{1}{2^{|X\setminus S|}} M(\es')$ (as it was the case for the uniform
space), because the entities do not all have the same
probability. This prevents us from being able to reduce to the
computation of~$\ssat(\cdot,\cdot,\cdot)$.  Instead, we use a
different definition of~$\H(\cdot,\cdot,\cdot)$, and prove that it can
directly be computed in a bottom-up fashion on the circuits.  We show
in Algorithm~\ref{algo:main} our algorithm to compute the SHAP score
for deterministic and decomposable Boolean circuits under product
distributions, which can be extracted from the proof in the
supplementary material.
Notice that by
using the techniques presented in Section \ref{sec:shapscore-d-Ds},
the first step of the algorithm transforms the input circuit $C$ into
an equivalent smooth circuit $D$ where each $\lor$-gate and
$\land$-gate has fan-in 2.

\begin{algorithm}
\caption{$\shap$-scores for deterministic and decomposable Boolean circuits}
\label{algo:main}
\SetKwBlock{Induction}{Compute values~$\gamma_g^{\ell}$ and~$\delta_g^{\ell}$ for every gate $g$ in $D$ and~$\ell\in \{0,\ldots,|\var(g) \setminus \{x\}|\}$ by bottom-up induction on~$D$ as follows:\label{line:begin}}{end\label{line:end}}
\SetKwInOut{Input}{Input}\SetKwInOut{Output}{Output}
\Input{Deterministic and decomposable Boolean circuit~$C$ over features~$X$ with ouput gate~$g_\out$, rational probability values~$\pr(x)$ for all~$x\in X$, entity~$\es\in \eset(X)$, and feature~$x\in X$.}
\Output{The value $\shap(C,\es,x)$ under the probability distribution~$\prd$.}
\BlankLine
\hrulealg
\BlankLine
\BlankLine
Transform $C$ into an equivalent smooth circuit $D$ where each $\lor$-gate and $\land$-gate has fan-in 2\;
\BlankLine
Compute the set~$\var(g)$ for every gate~$g$ in $D$\;
\BlankLine
\Induction{
    \uIf{$g$ is a constant gate with label~$a\in \{0,1\}$}{
        $\gamma_g^0,\, \delta_g^0 \leftarrow a$\;
    }
    \uElseIf{$g$ is a variable gate with~$\var(g)=\{x\}$}{
        $\gamma_g^0 \leftarrow 1$\;
        $\delta_g^0 \leftarrow 0$\;
    }
    \uElseIf{$g$ is a variable gate with~$\var(g)=\{y\}$ and $y\neq x$}{
        $\gamma_g^0,\, \delta_g^0 \leftarrow \pr(y)$\;
        $\gamma_g^1,\, \delta_g^1 \leftarrow \es(y)$\;
    }
    \uElseIf{$g$ is a $\lnot$-gate with input gate $g'$}{
        \For{$\ell\in \{0,\ldots,|\var(g)\setminus \{x\}|\}$}{
            $\gamma_g^{\ell} \leftarrow \binom{|\var(g) \setminus \{x\}|}{\ell} - \gamma_{g'}^{\ell}$\;
            $\delta_g^{\ell} \leftarrow \binom{|\var(g) \setminus \{x\}|}{\ell} - \delta_{g'}^{\ell}$\;
        }
    }

    \uElseIf{$g$ is an $\lor$-gate with input gates $g_1,g_2$\label{line:lor}}{
        \For{$\ell\in \{0,\ldots,|\var(g)\setminus \{x\}|\}$\label{line:for}}{
            $\gamma_g^{\ell} \leftarrow \gamma_{g_1}^{\ell} + \gamma_{g_2}^{\ell}$\;
            $\delta_g^{\ell} \leftarrow \delta_{g_1}^{\ell} + \delta_{g_2}^{\ell}$\;
        }\label{line:endlor}
    }
    \uElseIf{$g$ is an $\land$-gate with input gates $g_1,g_2$}{
        \For{$\ell\in \{0,\ldots,|\var(g)\setminus \{x\}|\}$}{
            $\gamma_g^{\ell} \leftarrow \sum_{\substack{\ell_1 \in \{0,\ldots,|\var(g_1)\setminus \{x\}|\}\\\ell_2 \in \{0,\ldots,|\var(g_2)\setminus \{x\}|\}\\ \ell_1 + \ell_2 = \ell}} \gamma_{g_1}^{\ell_1} \cdot \gamma_{g_2}^{\ell_2}$\;
            $\delta_g^{\ell} \leftarrow \sum_{\substack{\ell_1 \in \{0,\ldots,|\var(g_1)\setminus \{x\}|\}\\\ell_2 \in \{0,\ldots,|\var(g_2)\setminus \{x\}|\}\\ \ell_1 + \ell_2 = \ell}} \delta_{g_1}^{\ell_1} \cdot \delta_{g_2}^{\ell_2}$\;
        }
    }
}
\Return ${\displaystyle \sum_{k=0}^{|X|-1} \frac{k! \, (|X| - k - 1)!}{|X|!} \cdot \big[(\es(x) - \pr(x))(\gamma_{g_\out}^k - \delta_{g_\out}^k) \big]}$\label{line:return}\;
\end{algorithm}

%% file: discussion.tex
We leave open many interesting directions for future work. For
instance, we intend to extend our algorithm for efficiently computing
the $\shap$-score to work with non-Boolean classifiers, and to consider
more general probability distributions that could better capture
possible correlations and dependencies between features.  We also aim
to provide an experimental comparison of 
our algorithm,  but specialized for decision trees, with the
one provided in~\cite[Alg.~2]{lundberg2020local}.
Last, we intend to
test our algorithm on real-world scenarios.

%% file: app_KC.tex
In this appendix, we explain why binary decision trees and free binary decision diagrams (FBDDs) are special kinds of deterministic and decomposable Boolean circuits.
First we need to define these formalisms.

\paragraph*{Binary Decision Diagrams.}
A \emph{binary decision diagram} (BDD) over a set of variables
$X$ is a rooted directed acyclic graph~$D$ such that:
(i) each internal node is labeled with a variable
  from~$X$, and has exactly two outgoing edges: one
  labeled 0, the other one labeled 1; and
(ii) each leaf is labeled either~$0$ or~$1$.
Such a BDD represents a Boolean classifier in the following way.
Let~$\es$ be an entity over~$X$, and let $\pi_\es = u_1, \ldots,
u_m$ be the unique path in~$D$ satisfying the following conditions:
(a)~$u_1$ is the root of~$D$; (b)~$u_m$ is a leaf of~$D$; and (c) for
every~$i \in \{1, \ldots, m-1\}$, if the label of~$u_i$ is~$x \in X$,
then the label of the edge~$(u_i, u_{i+1})$ is equal to~$\es(x)$.
Then the {\em value of~$\es$ in~$D$}, denoted by~$D(\es)$, is defined
as the label of the leaf~$u_m$. Moreover, a binary decision
  diagram~$D$ is \emph{free} (FBDD) if for every path from the root to a
leaf, no two nodes on that path have the same label, and {\em a
  binary decision tree} is an FBDD whose underlying graph is a tree.\\

\noindent As we show next, FBDDs can be encoded in linear time as deterministic and decomposable Boolean circuits.

\paragraph*{Encoding FBDDs into deterministic and decomposable Boolean circuits (Folklore).}
Given an FBDD~$D$ over a set of variables~$X$, we explain how~$D$ can
be encoded as a deterministic and decomposable Boolean circuit~$C$
over~$X$. Notice that the technique used in this example also apply
to binary decision trees, as they are a particular case of FBDDs.  The
construction of~$C$ is done by traversing the structure of~$D$ in a
bottom-up manner. In particular, for every node~$u$ of~$D$, we
construct
a deterministic and decomposable circuit~$\alpha(u)$ that is
equivalent to the FBDD represented by the subgraph of~$D$ rooted
at~$u$. More precisely, for a leaf~$u$ of~$D$ that is labeled
with~$\ell \in \{0,1\}$, we define~$\alpha(u)$ to be the Boolean
circuit consisting of only one constant gate with label
$\ell$.
For an internal node~$u$ of~$D$ labeled with
variable~$x \in X$, let~$u_0$ and~$u_1$ be the nodes that we reach
from~$u$ by following the~$0$- and~$1$-labeled edge, respectively.
Then~$\alpha(u)$ is the Boolean circuit
depicted in the following figure:
\begin{center}
\begin{tikzpicture}
  \node[circ, minimum size=7mm] (n1) {{$\lor$}};
  \node[circw, below=5mm of n1] (n2) {};
  \node[circ, left=12mm of n2, minimum size=7mm] (n3) {{$\land$}}
  edge[arrout] (n1);
  \node[circ, right=12mm of n2, minimum size=7mm] (n4) {{$\land$}}
  edge[arrout] (n1);
  \node[circw, below=5mm of n3] (n5) {};
  \node[circ, left=4mm of n5, minimum size=7mm] (n6) {{$\neg$}}
  edge[arrout] (n3);
  \node[circw, right=4mm of n5, inner sep=-2] (n7) {{$\alpha(u_0)$}}
  edge[arrout] (n3);
  \node[circw, below=5mm of n4] (n8) {};
  \node[circ, left=4mm of n8, minimum size=7mm] (n9) {{$x$}}
  edge[arrout] (n4);
  \node[circw, right=4mm of n8, inner sep=-2] (n10) {{$\alpha(u_1)$}}
  edge[arrout] (n4);
  \node[circ, below=5mm of n6, minimum size=7mm] (n11) {{$x$}}
  edge[arrout] (n6);

\end{tikzpicture}
\end{center}

It is clear that the circuit that we obtain is equivalent to the input FBDD. We now argue that this circuit is deterministic and decomposable.
For the~$\lor$-gate shown in the figure, if an entity~$\es$ is
accepted by the Boolean circuit in its left-hand size, then~$\es(x) =
0$, while if an entity~$\es$ is accepted by the Boolean circuit in its
right-hand size, then~$\es(x) = 1$. Hence, we have that this
$\lor$-gate is deterministic, from which we conclude that~$\alpha(u)$
is deterministic, as~$\alpha(u_0)$ and~$\alpha(u_1)$ are also
deterministic by construction. Moreover, the~$\land$-gates shown in
the figure are decomposable as variable~$x$ is mentioned neither
in~$\alpha(u_0)$ nor in~$\alpha(u_1)$: this is because~$D$ is a
\emph{free} BDD. Thus, we conclude that $\alpha(u)$ is decomposable,
as~$\alpha(u_0)$ and~$\alpha(u_1)$ are decomposable by
construction. Finally, if~$u_{\mathsf{root}}$ is the root of~$D$, then
by construction we have that $\alpha(u_{\mathsf{root}})$ is a
deterministic and decomposable Boolean circuit equivalent to~$D$.
Note that this encoding can trivially be done in linear time.  Thus, we
often say, by abuse of terminology, that ``FBDDs (or binary decision trees)
are restricted kinds of deterministic and decomposable circuits".

%% file: app-proof-shapscore-d-Ds.tex
To complete the proof of Theorem \ref{thm:shapscore-d-Ds}, we need to
prove equation \eqref{eq:and-gate}. Recall that in this case, we have
that~$g$ is an~$\land$-gate, which is decomposable and has fan-in at
most 2. Moreover, we assume that~$g$ has exactly two input gates~$g_1$
and~$g_2$, and we fix~$k \in \{0,\ldots,|\var(g)|\}$.

To prove equation \eqref{eq:and-gate}, we need the following
notation. For two disjoint sets of variables~$X_1,X_2$ and
entities~$\es_1 \in \eset(X_1)$,~$\es_2 \in \eset(X_2)$, we denote
by~$\es_1 \cup \es_2$ the entity over~$X_1 \cup X_2$ that coincides
with~$\es_1$ over~$X_1$ and with~$\es_2$ over~$X_2$ (that is,~$\es_1
\cup \es_2 \in \eset(X_1 \cup X_2)$,~$(\es_1 \cup \es_2)(x_1) =
\es_1(x_1)$ for every~$x_1 \in X_1$, and~$(\es_1 \cup \es_2)(x_2) =
\es_2(x_2)$ for every~$x_2 \in X_2$). Moreover, for two sets~$S_1
\subseteq \eset(X_1)$,~$S_2 \subseteq \eset(X_2)$, we denote by~$S_1
\otimes S_2$ the set of entities over~$X_1\cup X_2$ defined as
\begin{align*}
S_1 \otimes S_2 \ \coloneqq \ \{\es_1\cup \es_2 \mid \es_1 \in
S_1 \text{ and } \es_2\in S_2\}.
\end{align*}
Given that~$g$ is a decomposable~$\land$-gate, we have that:
\begin{align*}
\sat(R_g) \ = \ \sat(R_{g_1}) \otimes \sat(R_{g_2}).
\end{align*}
Moreover, we have that~$\sat(R_g,\es_{|\var(g)},k) =
\sat(R_g) \cap \{ \es' \in \var(g) \mid |\sims(\es_{|\var(g)},\es')|=k\}$ and
\begin{align*}
&\big(\sat(R_{g_1}) \otimes \sat(R_{g_2})\big) \cap \{ \es' \in \var(g) \mid
|\sims(\es_{|\var(g)},\es')|=k\}\\
&\hspace{10pt} = \{\es_1\cup \es_2 \mid \es_1 \in \sat(R_{g_1}) \text{ and } \es_2 \in \sat(R_{g_2}) \} \cap \{ \es' \in \var(g) \mid |\sims(\es_{|\var(g)},\es')|=k\}\\
&\hspace{10pt} = \{\es_1\cup \es_2 \mid \es_1 \in \sat(R_{g_1}), \es_2 \in \sat(R_{g_2}), \text{ and } |\sims(\es_{|\var(g)},\es_1 \cup \es_2)|=k\} \\
&\hspace{10pt} = \{\es_1\cup \es_2 \mid \es_1 \in \sat(R_{g_1}), \es_2 \in \sat(R_{g_2}), \text{ and there exist } i \in \{0, \ldots, |\var(g_1)|\},\\
&\hspace{34pt} j \in \{0, \ldots, |\var(g_2)|\} \text{ such that } |\sims(\es_{|\var(g_1)},\es_1)|=i, |\sims(\es_{|\var(g_2)},\es_2)|=j, \text{ and } i+j = k\}\\
&\hspace{10pt} = \bigcup_{\substack{i \in \{0, \ldots, |\var(g_1)|\}\\ j \in \{0, \ldots, |\var(g_2)|\}\\ i+j=k}}
\{\es_1 \mid \es_1 \in \sat(R_{g_1}) \text{ and } |\sims(\es_{|\var(g_1)},\es_1)|=i\} \ \otimes\\[-27pt]
&\hspace{222pt} \{\es_2 \mid \es_2 \in \sat(R_{g_2}) \text{ and } |\sims(\es_{|\var(g_2)},\es_2)|=j\}\\[11pt]
&\hspace{10pt} = \bigcup_{\substack{i \in \{0, \ldots, |\var(g_1)|\}\\ j \in \{0, \ldots, |\var(g_2)|\}\\ i+j=k}}
\sat(R_{g_1}, \es_{|\var(g_1)}, i) \otimes
\sat(R_{g_2}, \es_{|\var(g_2)}, j).
\end{align*}
Combining the previous results, we obtain that
\begin{align*}
\sat(R_g,\es_{|\var(g)},k) \ = \ \bigcup_{\substack{i \in \{0, \ldots, |\var(g_1)|\}\\ j \in \{0, \ldots, |\var(g_2)|\}\\ i+j=k}}
\sat(R_{g_1}, \es_{|\var(g_1)}, i) \otimes
\sat(R_{g_2}, \es_{|\var(g_2)}, j).
\end{align*}
Thus, given that for every pair~$i_1, i_2 \in \{0, \ldots, |\var(g_1)|\}$
such that~$i_1 \neq i_2$, it holds that
\begin{align*}
\sat(R_{g_1}, \es_{|\var(g_1)}, i_1) \cap \sat(R_{g_1}, \es_{|\var(g_1)}, i_2) \ = \ \emptyset
\end{align*}
(and similarly for~$R_{g_2}$), we conclude by the definitions of~$\alpha_g^k$,~$\alpha_{g_1}^i$,~$\alpha_{g_2}^j$ that
\begin{align*}
\alpha_g^k \ = \ \sum_{\substack{i \in \{0, \ldots, |\var(g_1)|\}\\ j \in \{0, \ldots, |\var(g_2)|\}\\ i+j=k}} \alpha_{g_1}^i \cdot \alpha_{g_2}^j,
\end{align*}
which was to be shown.

%% file: app-proof-limits.tex
The validity of the equation from Lemma~\ref{lem:limits} will be consequence of the following property of the $\shap$-score:
for every Boolean classifier~$M$ over~$X$, entity~$\es\in \eset(X)$ and feature~$x\in X$, it holds that
\begin{align}
\label{eq:efficiency}
\sum_{x\in X} \shap(M,\es,x) \ = \ \phi(M,\es,X) - \phi(M,\es,\emptyset).
\end{align}
This property is often called the \emph{efficiency} property of the Shapley value. 
Although this is folklore,
we prove Equation~\eqref{eq:efficiency} here for
the reader's convenience.  For a permutation~$\pi:X\to \{1,\ldots,n\}$
and~$x\in X$, let~$S_\pi^x$ denote the set of features that appear
before~$x$ in~$\pi$.  Formally, $S_\pi^x \coloneqq \{y \in
X \mid \pi(y) < \pi(x)\}$. Then, letting~$\Pi(X)$ be the set of all
permutations~$\pi:X\to \{1,\ldots,n\}$, observe that the definition of~$\shap$-score from Definition~\ref{def:Shapley} can be rewritten as
\[
 \shap(M,\es,x) \ = \frac{1}{n!}\sum_{\pi \in \Pi(X)}  \big(\phi(M,\es,S_\pi^x \cup \{x\}) - \phi(M,\es,S_\pi^x)\big).
\]
Hence, we have that
\begin{align*}
\sum_{x\in X} \shap(M,\es,x)\ &= \ \frac{1}{n!} \sum_{x\in X} \ \sum_{\pi \in \Pi(X)}  \big(\phi(M,\es,S_\pi^x \cup \{x\}) - \phi(M,\es,S_\pi^x)\big)\\
&= \ \frac{1}{n!}\sum_{\pi \in \Pi(X)} \ \sum_{x\in X}  \big(\phi(M,\es,S_\pi^x \cup \{x\}) - \phi(M,\es,S_\pi^x)\big)\\
&= \ \frac{1}{n!}\sum_{\pi \in \Pi(X)} \big(\phi(M,\es,X) - \phi(M,\es,\emptyset)\big),
\end{align*}
where the last equality is obtained by noticing that the inner sum
is a telescoping sum.
This establishes Equation~\eqref{eq:efficiency}.
Now, we simply use the definition of~$\phi(\cdot,\cdot,\cdot)$ in this equation to obtain
\begin{align*}
\sum_{x\in X} \shap(M,\es,x)\ &= \ M(\es) - \frac{1}{2^n} \sum_{\es' \in \eset(X)} M(\es')\\
&= \ M(\es) - \frac{\ssat(M)}{2^n},
\end{align*}
thus proving Lemma~\ref{lem:limits}.

%% file: app-proof-shapscore-limits.tex
We have already
explained in the body of this article why Item (1) of Theorem~\ref{thm:shapscore-limits} holds. We now justify that (2) holds, by proving that from every 3-CNF formula $\psi$, we can build 
in polynomial time an equivalent deterministic Boolean circuit $C_\psi$.

Given a clause $\gamma = (\ell_1 \lor \ell_2 \lor \ell_3)$ consisting
of three literals, define $d(\gamma)$ as the propositional formula
\begin{multline*}
(\ell_1 \land \ell_2 \land \ell_3) \,\lor\,
(\overline{\ell_1} \land \ell_2 \land \ell_3) \,\lor\,
(\ell_1 \land \overline{\ell_2} \land \ell_3) \,\lor\,
(\ell_1 \land \ell_2 \land \overline{\ell_3}) \,\lor\,
(\overline{\ell_1} \land \overline{\ell_2} \land \ell_3) \,\lor\,
(\overline{\ell_1} \land \ell_2 \land \overline{\ell_3}) \,\lor\,
(\ell_1 \land \overline{\ell_2} \land \overline{\ell_3}), 
\end{multline*}
where $\overline{x} = \neg x$ and $\overline{\neg x} = x$, for each
propositional variable $x$. Clearly, $\gamma$ and $d(\gamma)$ are
equivalent formulae. Moreover, given a propositional formula $\psi
= \gamma_1 \land \cdots \land \gamma_k$ in 3-CNF, where each
$\gamma_i$ is a clause with three literals, define $d(\psi)$ as the
propositional formula $d(\gamma_1) \land \cdots \land
d(\gamma_k)$. Clearly, $\psi$ and $d(\psi)$ are equivalent formulae,
from which we have that $\ssat(\psi) = \ssat(d(\psi))$. Moreover,
$d(\psi)$ can be directly transformed into a deterministic Boolean
Circuit $C_{d(\psi)}$. Hence, from the fact that $C_{d(\psi)}$ can be
constructed in polynomial time from an input propositional formula
$\psi$ in 3-CNF, and the fact that $\ssat(\cdot)$ is $\shp$-hard for
3-CNFs, we have that Theorem~\ref{thm:shapscore-limits}~(2) holds
from Lemma \ref{lem:limits}.

%% file: app-proof-shapscore-d-Ds-prod.tex
In this section, we prove that computing the~$\shap$-score for Boolean classifiers
given as deterministic and decomposable Boolean circuits can be done in
polynomial time for product distributions; see
Theorem~\ref{thm:shapscore-d-Ds-prod} for the formal statement.  As mentioned
in Section~\ref{sec:prod}, the proof will be slightly more involved than that
of Theorem~\ref{thm:shapscore-d-Ds}; this is because not all entities have the
same probability, and this prevents us from reducing
to~$\ssat(\cdot,\cdot,\cdot)$.  Instead, we will use a different definition
of~$\H(\cdot,\cdot,\cdot)$ and show that it can directly be computed bottom-up
on the circuits. 

But before that, we introduce new notation that will be more convenient for
this proof.  For a Boolean classifier~$M$ over features~$X$, probability
distribution\footnote{Note that~$\mathcal{D}:\eset(X) \to [0,1]$ is actually a
\emph{probability mass function}, but we will abuse notation to simplify the
presentation.}~$\mathcal{D}:\eset(X) \to [0,1]$, entity~$\es\in \eset(X)$ and
set~$S\subseteq X$, we define

\begin{align*}
\phi_\mathcal{D}(M,\es,S) \ &\coloneqq \
\mathbb{E}_{\es' \sim \mathcal{D}}\big[M(\es') \mid \es' \in \asm(\es, S) \big].
\end{align*}
Notice that we now use the notation~$\mathbb{E}_{\es' \sim
\mathcal{D}}[f(\es')]$ for expected value of a random variable~$f$, instead of the
simpler~$\mathbb{E}[f(\es')]$ that we used in the body of the paper. This is because
we will sometimes need to make explicit what is the probability distribution to consider.
Then given a Boolean classifier~$M$ over a set of features~$X$, a probability
distribution~$\mathcal{D}$ over~$\eset(X)$, an entity~$\es$ over~$X$, and a
feature~$x \in X$, the {\em Shapley value of feature~$x$ in~$\es$ with respect
to~$M$ under~$\mathcal{D}$} is defined as

\begin{align}
\shap_\mathcal{D}(M,\es,x) \ \coloneqq \ \sum_{S \subseteq X\setminus\{x\}}
\frac{|S|! \, (|X| - |S| - 1)!}{|X|!} \, \bigg(\phi_\mathcal{D}(M, \es,S \cup \{x\})
- \phi_\mathcal{D}(M, \es,S)\bigg).
\end{align}

Note that by taking~$\mathcal{D}$ to be the uniform probability distribution
on~$\eset(X)$, we obtain the definition that we considered in
Section~\ref{sec:preliminaries}.  In this section we will consider the product
distributions~$\prd$ as defined in Section~\ref{sec:prod}. With these notation
in place, we can now start the proof.

For a Boolean classifier~$M$ over a set of variables~$X$, a probability distribution~$\mathcal{D}$ over~$\eset(X)$, an
entity~$\es\in \eset(X)$ and a natural number~$k \leq |X|$, define
\[\H_{\mathcal{D}}(M,\es, k) \ \coloneqq \ \sum_{\substack{S \subseteq X\\|S|=k}} \, \, \, \mathbb{E}_{\es' \sim \mathcal{D}}[M(\es') \mid \es' \in \asm(\es,S)].\]

Our proof of Theorem~\ref{thm:shapscore-d-Ds-prod} is
divided into two modular parts. The first part, which is developed in
Section \ref{subsec:shap-to-H}, consists in showing that the
problem of computing~$\shap_{\Pi_\cdot}(\cdot,\cdot,\cdot)$ can be reduced in
polynomial time to that of computing~$\H_{\Pi_\cdot}(\cdot,\cdot,\cdot)$. This
part of the proof is again a sequence of formula manipulations, and it only
uses the fact that deterministic and decomposable circuits can
be efficiently conditioned on a variable value.
In the second part of the proof, which is developed in
Section \ref{subsec:H}, we show that
computing~$\H_{\Pi_\cdot}(\cdot,\cdot,\cdot)$ can be done in polynomial time
for deterministic and decomposable Boolean circuits.
It is in this part that the magic of deterministic and decomposable
circuits really operates.

\subsection{Reducing in polynomial-time from~$\shap_{\Pi_\cdot}(\cdot,\cdot,\cdot)$ to~$\H_{\Pi_\cdot}(\cdot,\cdot,\cdot)$}
\label{subsec:shap-to-H}

In this section, we show that for deterministic and decomposable
Boolean circuits and under product distributions, the computation of the~$\shap$-score can be reduced
in polynomial time to the computation
of~$\H_{\Pi_\cdot}(\cdot,\cdot,\cdot)$. 
We wish to compute~$\shap_{\Pi_p}(C,\es,x)$, for a given deterministic and decomposable
circuit~$C$ over a set of variables~$X$, probability mapping~$p:X\to [0,1]$,
entity~$\es\in \eset(X)$ and feature~$x\in X$. Define
\[\diff_k(C,\es,x) \ \coloneqq \ \sum_{\substack{S\subseteq X\setminus \{x\}\\|S|=k}} (\phi_{\Pi_p}(C,\es,S\cup \{x\}) - \phi_{\Pi_p}(C,\es,S)),\]
and let~$n = |X|$. We then have
\begin{align*}
\shap_{\Pi_p}(C,\es,x) \ &= \ \sum_{S \subseteq X\setminus \{x\}} \frac{|S|!(n-|S|-1)!}{n!}(\phi_{\Pi_p}(C,\es,S\cup \{x\}) - \phi_{\Pi_p}(C,\es,S))\\
&= \ \sum_{k=0}^{n-1} \frac{k!(n-k-1)!}{n!} \diff_k(C,\es,x).
\end{align*}
\noindent Therefore, it is enough to show how to compute in polynomial time the
quantities~$\diff_k(C,\es,x)$ for each~$k\in \{0,\ldots,n-1\}$. By definition
of~$\phi_{_{\Pi_\cdot}}(\cdot,\cdot,\cdot)$ we have that
\begin{align}
\label{eq-diffk-H}
\tag{\dag}
\diff_k(C,\es,x) \ =& \ \bigg[\sum_{\substack{S\subseteq X\setminus \{x\}\\|S|=k}} \mathbb{E}_{\es' \sim \Pi_p}[C(\es') \mid \es' \in \asm(\es,S\cup \{x\})]\bigg]\\
\notag
& \hspace{100pt} - \bigg[\sum_{\substack{S\subseteq X\setminus \{x\}\\|S|=k}} \mathbb{E}_{\es' \sim \Pi_p}[C(\es') \mid \es' \in \asm(\es,S)]\bigg].
\end{align}
\noindent In this expression, let~$\alpha$ and~$\beta$ be the left- and
right-hand side terms in the subtraction. For a set of features~$X$,
mapping~$p:X \to [0,1]$ and~$S \subseteq X$, we write~$p_{|S}:S\to [0,1]$ for
the mapping that is the restriction of~$p$ to~$S$, and~$\Pi_{p_{|S}}: \eset(S)
\to [0,1]$ for the corresponding product distribution on~$\eset(S)$. Looking closer at~$\beta$, we
have that

\begin{align*}
\beta \ =& \ \sum_{\substack{S\subseteq X\setminus \{x\}\\|S|=k}} \mathbb{E}_{\es' \sim \Pi_p}[C(\es') \mid \es' \in \asm(\es,S)]\\
=& \ \, p(x) \cdot \sum_{\substack{S\subseteq X\setminus \{x\}\\|S|=k}} \mathbb{E}_{\es' \sim \Pi_p}[C(\es') \mid \es' \in \asm(\es,S) \text{ and $\es'(x)=1$}]\\
& \hspace{100pt} + (1-p(x)) \cdot \sum_{\substack{S\subseteq X\setminus \{x\}\\|S|=k}} \mathbb{E}_{\es' \sim \Pi_p}[C(\es') \mid \es' \in \asm(\es,S) \text{ and $\es'(x)=0$}]\\
=& \ \, p(x) \cdot \sum_{\substack{S\subseteq X\setminus \{x\}\\|S|=k}} \mathbb{E}_{\es'' \sim \Pi_{p_{|X\setminus \{x\}}}} [C_{+x}(\es'') \mid \es'' \in \asm(\es_{|X\setminus \{x\}},S)]\\
& \hspace{100pt} + (1-p(x)) \cdot \sum_{\substack{S\subseteq X\setminus \{x\}\\|S|=k}} \mathbb{E}_{\es'' \sim \Pi_{p_{|X\setminus \{x\}}}} [C_{-x}(\es'') \mid \es'' \in \asm(\es_{|X\setminus \{x\}},S)]\\
=& \ \, p(x) \cdot \H_{\Pi_{p_{|X\setminus \{x\}}}}(C_{+x},\es_{|X\setminus \{x\}},k) \, + \, (1-p(x)) \cdot \H_{\Pi_{p_{|X\setminus \{x\}}}}(C_{-x},\es_{|X\setminus \{x\}},k),
\end{align*}

\noindent where the last equality is obtained simply by using the definition
of~$\H_\cdot(\cdot,\cdot,\cdot)$. Hence, if we could compute in polynomial
time~$H_{\Pi_\cdot}(\cdot,\cdot,\cdot)$ for deterministic and decomposable Boolean
circuits, then we could compute~$\beta$ in polynomial time as~$C_{+x}$
and~$C_{-x}$ can be computed in linear time from~$C$, and they
are deterministic and decomposable Boolean circuits as well.
We now inspect the term~$\alpha$, which we recall is
\begin{align*}
\alpha \ =& \ \sum_{\substack{S\subseteq X\setminus \{x\}\\|S|=k}} \mathbb{E}_{\es' \sim \Pi_p}[C(\es') \mid \es' \in \asm(\es,S\cup \{x\})].
\end{align*}

\noindent But then observe that, for~$S\subseteq X\setminus \{x\}$ and~$\es' \in \asm(\es,S\cup \{x\})$, it holds that
\begin{equation*}
C(\es') \ =
\begin{cases}
C_{+x}(\es'_{|X\setminus\{x\}}) & \text{if } \es(x)=1\\
C_{-x}(\es'_{|X\setminus\{x\}}) & \text{if } \es(x)=0
\end{cases}.
\end{equation*}

\noindent Therefore, if~$\es(x)=1$, we have that

\begin{align*}
\alpha \ =& \ \sum_{\substack{S\subseteq X\setminus \{x\}\\|S|=k}} \mathbb{E}_{\es'' \sim \Pi_{p_{|X\setminus \{x\}}}}[C_{+x}(\es'') \mid \es'' \in \asm(\es_{|X\setminus \{x\}},S)]\\
 =& \ \H_{\Pi_{p_{|X\setminus \{x\}}}}(C_{+x},\es_{|X\setminus \{x\}},k)
\end{align*}
whereas if~$\es(x)=0$, we have that
\begin{align*}
\alpha \ = \ \H_{\Pi_{p_{|X\setminus \{x\}}}}(C_{-x},\es_{|X\setminus \{x\}},k).
\end{align*}
Hence, again, if we were able to compute in polynomial
time~$\H_{\Pi_\cdot}(\cdot,\cdot,\cdot)$ for deterministic and decomposable Boolean
circuits, then we could compute~$\alpha$ in polynomial time (as
deterministic and decomposable Boolean circuits~$C_{+x}$ and~$C_{-x}$
can be computed in linear time from~$C$).
But then we deduce from~\eqref{eq-diffk-H} that~$\diff_k(C,\es,x)$ could
be computed in polynomial time for each~$k\in \{0,\ldots,n-1\}$, from
which we have that~$\shap_{\Pi_p}(C,\es,x)$ could be computed in polynomial
time, therefore concluding the existence of the reduction claimed in this section.

\subsection{Computing~$\H_{\Pi_\cdot}(\cdot,\cdot,\cdot)$ in polynomial time}
\label{subsec:H}

We now take care of the second part of the proof of
Theorem~\ref{thm:shapscore-d-Ds-prod}, i.e., proving that
computing~$\H_{\Pi_\cdot}(\cdot,\cdot,\cdot)$ for deterministic and
decomposable Boolean circuits can be done in polynomial time. Formally:

\begin{lemma}
\label{lem:H}
The following problem can be solved in polynomial time. Given as input
a deterministic and decomposable Boolean circuit~$C$ over
a set of variables~$X$, rational probability values~$p(x)$ for each~$x\in X$,
an entity~$\es\in \eset(X)$ and a natural number ~$k \leq |X|$, compute
the quantity~$\H_{\Pi_p}(C,\es,k)$.
\end{lemma}

We first perform two preprocessing steps on~$C$, which will simplify the proof. These are the same preprocessing steps that we did in Section~\ref{sec:shapscore-d-Ds},
but we added more details for the reader's convenience.

\begin{description}
    \item[Rewriting to fan-in at most 2.] First, we modify the
		circuit~$C$ so that the fan-in of every~$\lor$-
		and~$\land$-gate is at most~$2$. This can simply be
		done in linear time by rewriting every~$\land$-gate (resp.,
		and~$\lor$-gate) of fan-in~$m > 2$ with a chain
		of~$m-1$ $\land$-gates (resp.,~$\lor$-gates) of
		fan-in~$2$. It is
		clear that the resulting Boolean circuit is
		deterministic and decomposable. Hence, from now on we
		assume that the fan-in of every~$\lor$-
		and~$\land$-gate of~$C$ is at most~$2$.

\item[Smoothing the circuit.] Recall that a deterministic and decomposable circuit~$C$
		is \emph{smooth} if for
		every~$\lor$-gate~$g$ and input gates~$g_1,g_2$
		of~$g$, we have that~$\var(g_1) = \var(g_2)$,
		and we call such an~$\lor$-gate smooth. We modify as
		follows the circuit~$C$ so that it becomes smooth.
		Recall that by the previous paragraph, we assume that
		the fan-in of every~$\lor$-gate is at most~$2$.  For
		an~$\lor$-gate~$g$ of~$C$ having two input
		gates~$g_1,g_2$ violating the smoothness condition,
		define $S_1 \coloneqq \var(g_1) \setminus \var(g_2)$
		and~$S_2 \coloneqq \var(g_2) \setminus \var(g_1)$, and
		let~$d_{S_1}$, $d_{S_2}$ be Boolean circuits defined
		as follows. If~$S_1 = \emptyset$, then~$d_{S_1}$
		consist of the single constant gate~$1$. Otherwise,
		$d_{S_1}$ encodes the propositional formula
		$\land_{x\in S_1} (x \lor \lnot x)$, but it is
		constructed in such a way that every~$\land$- and
		$\lor$-gate has fan-in at most~$2$. Boolean circuit
		$d_{S_2}$ is constructed exactly as~$d_{S_1}$ but
		considering the set of variables~$S_2$ instead of
		$S_1$.  Observe that~$\var(d_{S_1})=S_1$,
		$\var(d_{S_2})=S_2$ and~$d_{S_1}$,~$d_{S_2}$ always
		evaluate to~$1$.  Then, we transform~$g$ into a smooth
		$\lor$-gate by replacing gate~$g_1$ by a
		decomposable~$\land$-gate~$(g_1 \land d_{S_2})$, and
		gate~$g_2$ by a decomposable~$\land$-gate~$(g_2 \land
		d_{S_1})$. This does not change the Boolean classifier
		computed.  Moreover, since~$\var(g_1 \land d_{S_2})
		= \var(g_2 \land d_{S_1}) = \var(g_1) \cup \var(g_2)$,
		we have that~$g$ is now smooth. Finally, the resulting
		Boolean circuit is deterministic and
		decomposable. Hence, by repeating the previous
		procedure for each non-smooth~$\lor$-gate, we conclude
		that~$C$ can be transformed into an equivalent smooth
		Boolean circuit in polynomial time, which is
		deterministic and decomposable, and where each gate
		has fan-in at most 2. Thus, from now on we also assume
		that~$C$ is smooth.
\end{description}

\begin{proof}[Proof of Lemma \ref{lem:H}]
Let~$C$ be a deterministic and decomposable Boolean circuit~$C$ over a set of
variables~$X$, $p:X \to [0,1]$ be a rational probability mapping, $\es\in
\eset(X)$ and $k$ a natural number such that~$k \leq |X|$, and let~$n = |X|$.
For a gate~$g$ of~$C$, let~$R_g$ be the Boolean circuit over~$\var(g)$ that is
defined by considering the subgraph of~$C$ induced by the set of gates~$g'$
in~$C$ for which there exists a path from~$g'$ to~$g$ in~$C$.\footnote{The only
difference between~$R_g$ and~$C_g$ (defined in Section~\ref{sec:preliminaries})
is that we formally regard~$R_g$ as a Boolean classifier over~$\var(g)$, while
we formally regarded~$C_g$ as a Boolean classifier over~$X$.} Notice that~$R_g$
is a deterministic and decomposable Boolean circuit with output gate~$g$.
Moreover, for a gate~$g$ and natural number~$l \leq |\var(g)|$,
define~$\alpha_g^l \coloneqq \H_{\Pi_{p_{|\var(g)}}}(R_g,\es_{|\var(g)},l)$, which we recall is equal, by definition, to
\[ \H_{\Pi_{p_{|\var(g)}}}(R_g,\es_{|\var(g)},l) = \sum_{\substack{S \subseteq \var(g)\\|S|=l}} \, \, \mathbb{E}_{\es' \sim \Pi_{p_{|\var(g)}}}[R_g(\es') \mid \es' \in \asm(\es_{|\var(g)},S)].\]
We will show how to
compute all the values~$\alpha_g^l$ for every gate~$g$ of~$C$ and~$l\in
\{0,\ldots,|\var(g)|\}$ in polynomial time. This will conclude the proof since,
for the output gate~$g_\out$ of~$C$, we have that~$\alpha_{g_\out}^k =
\H_{\Pi_p}(C,\es,k)$.  Next we explain how to compute these values by bottom-up
induction on~$C$.

\begin{description}
    \item[Variable gate.] $g$ is a variable gate with label~$y \in X$,
        so that~$\var(g)=\{y\}$. Then for~$\es' \in \eset(\{y\})$ we have~$R_g(\es')=\es'(y)$, therefore
		\begin{align*}
		\alpha^0_g &= \ \sum_{\substack{S \subseteq \{y\}\\|S|=0}} \, \, \mathbb{E}_{\es' \sim \Pi_{p_{|\{y\}}}}[\es'(y) \mid \es' \in \asm(\es_{|\{y\}},S)]\\
		&= \ \mathbb{E}_{\es' \sim \Pi_{p_{|\{y\}}}}[\es'(y) \mid \es' \in \asm(\es_{|\{y\}},\emptyset)]\\
		&= \ \mathbb{E}_{\es' \sim \Pi_{p_{|\{y\}}}}[\es'(y)]\\
		&= \ 1 \cdot p(y) + 0 \cdot (1-p(y))\\
		&= \ p(y)
		\end{align*}
		\noindent and
		\begin{align*}
		\alpha^1_g &= \ \sum_{\substack{S \subseteq \{y\}\\|S|=1}} \, \, \mathbb{E}_{\es' \sim \Pi_{p_{|\{y\}}}}[\es'(y) \mid \es' \in \asm(\es_{|\{y\}},S)]\\
		&= \ \mathbb{E}_{\es' \sim \Pi_{p_{|\{y\}}}}[\es'(y) \mid \es' \in \asm(\es_{|\{y\}},\{y\})]\\
		&= \ \es(y).
		\end{align*}

    \item[Constant gate.] $g$ is a constant gate with
		label~$a\in \{0,1\}$, and~$\var(g) = \emptyset$. We recall the mathematical
		convention that there is a unique function with the
		empty domain and, hence, a unique entity
		over~$\emptyset$. But then
		\begin{align*}
		\alpha_g^0 &= \ \sum_{\substack{S \subseteq \emptyset \\|S|=0}} \, \, \mathbb{E}_{\es' \sim \Pi_{p_{|\emptyset}}}[a \mid \es' \in \asm(\es_{|\emptyset},S)]\\
		&= \  \mathbb{E}_{\es' \sim \Pi_{p_{|\emptyset}}}[a \mid \es' \in \asm(\es_{|\emptyset},\emptyset)]\\
		&= \ a.
		\end{align*}
    \item[$\lnot$-gate.] $g$ is a~$\lnot$-gate with input
			gate~$g'$. Notice
			that~$\var(g)=\var(g')$. Then, since for~$\es' \in \eset(\var(g))$ we have that~$R_g(\es') = 1- R_{g'}(\es')$, we have
            \begin{align*}
            \alpha_g^l \ &= \ \sum_{\substack{S \subseteq \var(g)\\|S|=l}} \, \, \mathbb{E}_{\es' \sim \Pi_{p_{|\var(g)}}}[1- R_{g'}(\es') \mid \es' \in \asm(\es_{|\var(g)},S)].
            \end{align*}
            \noindent By linearity of expectations we deduce that
            \begin{align*}
             \alpha_g^l \ &= \ \sum_{\substack{S \subseteq \var(g)\\|S|=l}} \, \, \mathbb{E}_{\es' \sim \Pi_{p_{|\var(g)}}}[1\mid \es' \in \asm(\es_{|\var(g)},S)]\\
            & \hspace{60pt} - \sum_{\substack{S \subseteq \var(g)\\|S|=l}} \, \, \mathbb{E}_{\es' \sim \Pi_{p_{|\var(g)}}}[R_{g'}(\es')\mid \es' \in \asm(\es_{|\var(g)},S)]\\
            &= \ \big[\sum_{\substack{S \subseteq \var(g)\\|S|=l}} 1\big] - \alpha_{g'}^l\\
            &= \ \binom{|\var(g)|}{l} - \alpha_{g'}^l
            \end{align*}
                        for every~$l \in \{0, \ldots, |\var(g)|\}$.
			By induction, the values~$\alpha_{g'}^l$ for~$l\in\{0,\ldots,|\var(g)|\}$ have
			already been computed. 
            Thus, we can compute all the values~$\alpha_g^l$ for~$l\in
			\{0,\ldots,|\var(g)|\}$ in polynomial time.
\item[$\lor$-gate.] $g$ is an~$\lor$-gate. By assumption, recall that~$g$ is deterministic, smooth and has fan-in at most 2.
If~$g$ has only one input~$g'$, then clearly~$\var(g)=\var(g')$ and~$\alpha_g^l = \alpha_{g'}^l$ for
every~$l\in \{0,\ldots,|\var(g)|\}$. Thus, assume that~$g$ has exactly
two input gates~$g_1$ and~$g_2$, and recall that~$\var(g_1)
= \var(g_2) = \var(g)$, because~$g$ is smooth. Given that~$g$
is deterministic, observe that for every~$\es' \in \eset(\var(g))$ we have~$R_g(\es') = R_{g_1}(\es') + R_{g_1}(\es')$.
But then for~$l\in \{0,\ldots,|\var(g)|\}$ we have
\begin{align*}
\alpha_g^l &= \ \sum_{\substack{S \subseteq \var(g)\\|S|=l}} \, \, \mathbb{E}_{\es' \sim \Pi_{p_{|\var(g)}}}[R_{g_1}(\es') + R_{g_2}(\es') \mid \es' \in \asm(\es_{|\var(g)},S)]\\
&= \ \sum_{\substack{S \subseteq \var(g)\\|S|=l}} \, \, \mathbb{E}_{\es' \sim \Pi_{p_{|\var(g)}}}[R_{g_1}(\es') \mid \es' \in \asm(\es_{|\var(g)},S)]\\
& \hspace{60pt} + \sum_{\substack{S \subseteq \var(g)\\|S|=l}} \, \, \mathbb{E}_{\es' \sim \Pi_{p_{|\var(g)}}}[R_{g_2}(\es') \mid \es' \in \asm(\es_{|\var(g)},S)]\\
&= \alpha_{g_1}^l + \alpha_{g_2}^l,
\end{align*}
where the second equality is by linearity of the expectation, and the last equality is valid because~$g$ is smooth.
By induction, the
values~$\alpha_{g_1}^l$ and~$\alpha_{g_2}^l$, for
each~$l\in\{0,\ldots,|\var(g)|\}$, have already been
computed. Therefore, we can compute all the values~$\alpha_g^l$
for~$l\in \{0,\ldots,|\var(g)|\}$ in polynomial time.

\item[$\land$-gate.] $g$ is an~$\land$-gate. By assumption, recall that~$g$ is decomposable and has fan-in at most 2.
If~$g$ has only one input~$g'$, then clearly~$\var(g)=\var(g')$ and~$\alpha_g^l = \alpha_{g'}^l$ for
every~$l\in \{0,\ldots,|\var(g)|\}$. Thus, assume that~$g$ has exactly
two input gates~$g_1$ and~$g_2$. For~$\es' \in \eset(\var(g))$ we have that~$R_g(\es') = R_{g_1}(\es'_{|\var(g_1)}) \cdot R_{g_2}(\es'_{|\var(g_2)})$.
Moreover, since~$\var(g) = \var(g_1) \cup \var(g_2)$ and~$\var(g_1)\cap \var(g_2) = \emptyset$ (because~$g$ is decomposable), observe that every~$S\subseteq \var(g)$ can be uniquely decomposed into~$S_1 \subseteq \var(g_1)$, $S_2 \subseteq \var(g_2)$ such that~$S = S_1\cup S_2$. Henceforth, for~$l\in \{0,\ldots,|\var(g)|\}$ we have
\begin{align*}
\alpha_g^l &= \ \sum_{\substack{S_1 \subseteq \var(g_1)\\|S_1| \leq l}} \, \sum_{\substack{S_2 \subseteq \var(g_2)\\|S_2| = |\var(g)| - |S_1|}}\, \mathbb{E}_{\es' \sim \Pi_{p_{|\var(g)}}}[R_{g_1}(\es'_{|\var(g_1)}) \cdot R_{g_2}(\es'_{|\var(g_2)}) \mid \es' \in \asm(\es_{|\var(g)},S_1\cup S_2)].\\
\end{align*}
But, by definition of the product distribution~$\Pi_{p_{|\var(g)}}$, we have that~$R_{g_1}(\es'_{|\var(g_1)})$ and~$R_{g_2}(\es'_{|\var(g_2)})$ are independent random variables, hence we deduce
\begin{align*}
\alpha_g^l &= \ \sum_{\substack{S_1 \subseteq \var(g_1)\\|S_1| \leq l}} \, \sum_{\substack{S_2 \subseteq \var(g_2)\\|S_2| = |\var(g)| - |S_1|}}\, \bigg[\mathbb{E}_{\es' \sim \Pi_{p_{|\var(g)}}}[R_{g_1}(\es'_{|\var(g_1)})\mid \es' \in \asm(\es_{|\var(g)},S_1\cup S_2)]\\
& \hspace{110pt} \times \mathbb{E}_{\es' \sim \Pi_{p_{|\var(g)}}}[R_{g_2}(\es'_{|\var(g_2)})\mid \es' \in \asm(\es_{|\var(g)},S_1\cup S_2)]\bigg].\\
&= \ \sum_{\substack{S_1 \subseteq \var(g_1)\\|S_1| \leq l}} \, \sum_{\substack{S_2 \subseteq \var(g_2)\\|S_2| = |\var(g)| - |S_1|}}\, \bigg[\mathbb{E}_{\es'' \sim \Pi_{p_{|\var(g_1)}}}[R_{g_1}(\es'')\mid \es'' \in \asm(\es_{|\var(g_1)},S_1)]\\
& \hspace{110pt} \times \mathbb{E}_{\es'' \sim \Pi_{p_{|\var(g_2)}}}[R_{g_2}(\es'')\mid \es'' \in \asm(\es_{|\var(g_2)},S_2)]\bigg],
\end{align*}
where the last equality is simply by definition of the product distributions, and because~$R_{g_1}(\es'_{|\var(g_1)})$ is independent of the value~$\es'_{|\var(g_2)}$, and similarly for~$R_{g_2}(\es'_{|\var(g_2)})$. But then
\begin{align*}
\alpha_g^l &= \ \sum_{\substack{S_1 \subseteq \var(g_1)\\|S_1| \leq l}} \, \mathbb{E}_{\es'' \sim \Pi_{p_{|\var(g_1)}}}[R_{g_1}(\es'')\mid \es'' \in \asm(\es_{|\var(g_1)},S_1)] \, \times \sum_{\substack{S_2 \subseteq \var(g_2)\\|S_2| = |\var(g)| - |S_1|}}\, \\
& \hspace{200pt} \mathbb{E}_{\es'' \sim \Pi_{p_{|\var(g_2)}}}[R_{g_2}(\es'')\mid \es'' \in \asm(\es_{|\var(g_2)},S_2)]\\
&= \ \sum_{\substack{S_1 \subseteq \var(g_1)\\|S_1| \leq l}} \, \mathbb{E}_{\es'' \sim \Pi_{p_{|\var(g_1)}}}[R_{g_1}(\es'')\mid \es'' \in \asm(\es_{|\var(g_1)},S_1)] \times  \alpha_{g_2}^{|\var(g)|-|\var(S_1)|}\\
&= \sum_{l_1 = 0}^l \alpha_{g_2}^{|\var(g)|-l_1} \times \sum_{\substack{S_1 \subseteq \var(g_1)\\|S_1| = l_1}} \, \mathbb{E}_{\es'' \sim \Pi_{p_{|\var(g_1)}}}[R_{g_1}(\es'')\mid \es'' \in \asm(\es_{|\var(g_1)},S_1)]\\
&= \sum_{l_1 = 0}^l \alpha_{g_2}^{|\var(g)|-l_1} \times \alpha_{g_1}^{l_1}\\
&= \sum_{\substack{l_1 \in \{0,\ldots,|\var(g_1)|\}\\l_2 \in \{0,\ldots,|\var(g_2)|\}\\ l_1 + l_2 = l}} \alpha_{g_1}^{l_1} \cdot \alpha_{g_2}^{l_2}.
\end{align*}

By induction, the values~$\alpha_{g_1}^{l_1}$ and~$\alpha_{g_2}^{l_2}$, for
each~$l_1 \in\{0,\ldots,|\var(g_1)|\}$
and~$l_2 \in \{0,\ldots,|\var(g_2)|\}$, have already been
computed. Therefore, we can compute all the values~$\alpha_g^l$
for~$l\in \{0,\ldots,|\var(g)|\}$ in polynomial time.
\end{description}
This concludes
the proof of Lemma~\ref{lem:H} and, hence, the proof of
Theorem~\ref{thm:shapscore-d-Ds-prod}.
\end{proof}